\documentclass[letterpaper,11pt]{article}
\usepackage[margin=1in]{geometry}
\usepackage[utf8]{inputenc} 
\usepackage[T1]{fontenc}    
\usepackage[colorlinks=true,breaklinks=true,bookmarks=true,urlcolor=blue,citecolor=blue,linkcolor=blue,bookmarksopen=false,draft=false]{hyperref}
\usepackage{url}            
\usepackage{booktabs}       
\usepackage{multirow}
\usepackage{microtype}      
\usepackage{amsfonts,amsmath, amsthm,mathtools,stmaryrd}
\usepackage{xcolor}
\usepackage[inline]{enumitem}
\usepackage{appendix}
\usepackage{float, graphicx, subcaption}

\usepackage{natbib}
\renewcommand{\cite}[1]{\citep{#1}}

\newcommand{\XCal}{\mathcal{X}}
\newcommand{\YCal}{\mathcal{Y}}

\newcommand{\SCal}{\mathcal{S}}
\newcommand{\RCal}{\mathcal{R}}

\newcommand{\LL}{\mathcal{L}}
\newcommand{\pbm}{\mathbb{P}}
\newcommand{\T}{\mathbb{T}}
\newcommand{\E}{\mathbb{E}}
\newcommand{\Ef}{\mathcal{E}}

\newcommand{\R}{\mathbb{R}}
\newcommand{\X}{\mathcal X}

\newcommand{\ICal}{\mathcal I}

\newcommand{\CCal}{\mathcal C}

\newtheorem{rmk}{Remark}[section]

\newtheorem{thm}{Theorem}[section]

\newtheorem{prop}{Proposition}[section]
\newtheorem{defn}{Definition}[section]

\newtheorem{asp}{Assumption}[section]

\def\argmin_#1{\underset{#1}{\mathrm{arg\,min\, }}}
\def\argmax_#1{\underset{#1}{\mathrm{arg\,max\, }}}

\makeatletter
\def\dasharrowfill@#1#2#3#4{%
        $\m@th
        \thickmuskip0mu
        \medmuskip\thickmuskip
        \thinmuskip\thickmuskip
        \relax
        #4#1\mkern2mu
        \xleaders\hbox{$#4\mkern2mu#2\mkern2mu$}\hfill
        \mkern2mu
        #3$%
}

\def\dashleftarrowfill@{\dasharrowfill@\leftarrow\relbar\relbar}
\def\dashrightarrowfill@{\dasharrowfill@\relbar\relbar\rightarrow}
\def\dashleftrightarrowfill@{\dasharrowfill@\leftarrow\relbar\rightarrow}
\def\dashLeftarrowfill@{\dasharrowfill@\Leftarrow\Relbar\Relbar}
\def\dashRightarrowfill@{\dasharrowfill@\Relbar\Relbar\Rightarrow}
\def\dashLeftrightarrowfill@{\dasharrowfill@\Leftarrow\Relbar\Rightarrow}

\providecommand*\xdashleftarrow[2][]{%
  \ext@arrow 0055{\dashleftarrowfill@}{#1}{#2}}
\providecommand*\xdashrightarrow[2][]{%
  \ext@arrow 0055{\dashrightarrowfill@}{#1}{#2}}
\providecommand*\xdashleftrightarrow[2][]{%
  \ext@arrow 0055{\dashleftrightarrowfill@}{#1}{#2}}
\providecommand*\xdashLeftarrow[2][]{%
  \ext@arrow 0055{\dashLeftarrowfill@}{#1}{#2}}
\providecommand*\xdashRightarrow[2][]{%
  \ext@arrow 0055{\dashRightarrowfill@}{#1}{#2}}
\providecommand*\xdashLeftrightarrow[2][]{%
  \ext@arrow 0055{\dashLeftrightarrowfill@}{#1}{#2}}
\makeatother

\begin{document}

\title{Feasibility and Transferability of Transfer Learning: A Mathematical Framework}

\author{
Haoyang Cao
\thanks{Centre de Math\'ematiques Appliqu\'ees, Ecole Polytechnique.
\textbf{Email:}
haoyang.cao@polytechnique.edu}
\and
Haotian Gu
\thanks{Department of Mathematics, UC Berkeley.
\textbf{Email:} 
haotian$\_$gu@berkeley.edu }
\and
Xin Guo
\thanks{Department of Industrial Engineering \& Operations Research, UC Berkeley.
\textbf{Email:} 
xinguo@berkeley.edu}
\and
Mathieu Rosenbaum
\thanks{Centre de Math\'ematiques Appliqu\'ees, Ecole Polytechnique.
\textbf{Email:}
mathieu.rosenbaum@polytechnique.edu}
}
\date{January 26, 2023}
\maketitle

\begin{abstract}
Transfer learning is an emerging and popular paradigm for utilizing existing knowledge from  previous learning tasks to improve the performance of new ones. Despite its numerous empirical successes, theoretical analysis for transfer learning is limited. In this paper we build for the first time, to the best of our knowledge, a mathematical framework for the general procedure of transfer learning. Our unique reformulation of transfer learning as an optimization problem allows for the first time, analysis of its feasibility. Additionally, we propose a novel concept of {\it transfer risk} to evaluate transferability of transfer learning. Our numerical studies using the Office-31 dataset demonstrate the potential and benefits of incorporating transfer risk in the evaluation of transfer learning performance.
\end{abstract}

\section{Introduction}

The basic idea of transfer learning is simple: it is  to  leverage knowledge from a well-studied learning problem, known as the source task,  to improve the performance of a new learning problem  with similar features, known as the target task.  Transfer learning has seen success in a variety of field, including natural language processing \citep{ruder2019transfer, devlin-etal-2019-bert, sung2022vl}, sentiment analysis \cite{jiang2007instance, deng2013sparse, liu2019survey}, computer vision \cite{deng2009imagenet, long2015learning, ganin2016domain, wang2018deep}, activity recognition \cite{cook2013transfer, wang2018stratified}, medical data analysis \cite{zeng2019automatic, wang2022transfer, kim2022transfer}, bio-informatics \cite{hwang2010heterogeneous}, finance \cite{leal2020learning,rosenbaum2021deep}, recommendation system \cite{pan2010transfer, yuan2019darec}, and fraud detection \cite{lebichot2020deep}.
See also review papers \cite{survey1,tan2018survey,zhuang2020comprehensive}.
Transfer learning is a versatile and enduring paradigm in the rapidly changing AI landscape where new machine learning techniques and tools mushroom with a breakneck speed. 

Despite its empirical successes, studies on transfer learning are primarily based on  trial-and-error heuristics. Virtually there are neither  basic theoretical frameworks for the general procedure of transfer learning, nor studies on the fundamental issue of it feasibility.

\paragraph{Existing theoretical works of transfer learning.}  Earlier theoretical works for transfer learning tend to focus on specific learning problems, such as classification, and derive upper bounds of generalization error under different measurements. There are the VC-dimension of the hypothesis space adopted in \cite{blitzer2007learning},  total variation distance in \cite{ben2010theory}, $f$-divergence in \cite{harremoes2011pairs}, Jensen-Shannon divergence in \cite{zhao2019learning}, $\mathcal{H}$-score in \cite{bao2019information}, mutual information in \cite{bu2020tightening}, and more recently $\X^2$-divergence in \cite{tong2021mathematical}, and variations of optimal transport cost in \cite{tan2021otce}. 
 
Another line of theoretical studies interprets transferability for transfer learning as a measurement of similarity between the source and the target data using various divergences, such as low-rank common information in \cite{saenko2010adapting}, KL-divergence in \cite{ganin2015unsupervised,ganin2016domain,tzeng2017adversarial}, $l_2$-distance in \cite{long2014transfer}, and the optimal transport cost in \cite{flamary2016optimal}. 
  
\paragraph{Our work.}
In this paper, we address the issues of feasibility and transferability for transfer learning through rigorous and comprehensive mathematical analysis.  
\begin{itemize}
\item We build, for the first time to the  best of our knowledge, a mathematical framework for the general procedure of transfer learning, identifying its three key steps and components. 

\item We reformulate this three-step transfer learning procedure as an optimization problem, enabling us to analyze, for the first time, its feasibility. This is accomplished via analyzing the well-definedness of the corresponding optimization problem.

\item Additionally, we propose a novel concept of {\it transfer risk}  to evaluate the transferability of transfer learning. Our form of transfer risk accounts for {\it both} the compatibility between the output and the input data {\it and} the compatibility between the models  in the source and the target tasks,  allowing for the study of the trade-off between the two. This novel notion of transfer risk generalizes earlier works on transferability, including the $\mathcal{H}$-score proposed in  a particular classification setting 
in \cite{bao2019information} and  \cite{saenko2010adapting, ganin2016domain, long2014transfer} on the relation between source and target inputs.

\item In the special case of linear regression with
Gaussian data, we show that the regret in the learning problem can be lower bounded by Wasserstein-based transfer risk, which in turn is useful for prescreening unsuitable candidate pretrained models or source tasks.

\item Our numerical studies using the Office-31 dataset show the consistency of the transfer risk with existing statistical metrics in evaluating the performance of transfer learning; and demonstrate the potential and benefit of adopting transfer risk to improve computational efficiency.

\end{itemize}

\section{Mathematical Framework and Feasibility of Transfer Learning}
 
In this section, we will establish necessary concepts and a mathematical framework for the entire procedure of transfer learning. We will then reformulate transfer learning as an optimization problem, the well-definedness of which yields the feasibility of transfer learning.

For ease of exposition and without loss of generality, we will focus on a supervised setting, with a source task $S$ and a target task $T$  on a probability space $(\Omega,\mathcal{F},\pbm)$.

\subsection{Mathematical Framework for Transfer Learning}

\paragraph{Target task $T$.}
 In the target task $T$, denote $\XCal_T$ and $\YCal_T$ as its input and output spaces, respectively, and $(X_T,Y_T)$ as a pair of $\XCal_T\times\YCal_T$-valued random variables. Here, $(\XCal_{T},\|\cdot\|_{\XCal_{T}})$ and $(\YCal_{T},\|\cdot\|_{\YCal_{T}})$ are Banach spaces with norms $\|\cdot\|_{\XCal_{T}}$ and $\|\cdot\|_{\YCal_{T}}$, respectively.  Let $L_T:\YCal_T\times\YCal_T\to\R$ be a real-valued function, and assume that the learning objective for the target task is 
\begin{equation}\label{eq: obj-t}
    \min_{f\in A_T}\LL_T(f_T)=\min_{f_T\in A_T}\E[L_T(Y_T,f_T(X_T))],
\end{equation}
where $\LL_T(f_T)$ is a loss function that measures a model  $f_T:\XCal_T\to\YCal_T$ for the target task $T$, and $A_T$ denotes the set of target models such that
 \begin{equation}\label{eq: a-t}
A_T\subset 
\{f_T|f_T:\XCal_T\to\YCal_T\}.
\end{equation}

Take the image classification task as an example, $\XCal_T$ is a space containing images as high dimensional vectors, $\YCal_T$ is a space containing image labels,  $(X_T, Y_T)$ is a pair of random variables satisfying the empirical distribution of target images and their corresponding labels, and  $L_T$ is the cross-entropy loss function between the actual label $Y_T$ and the predicted label $f_T(X_T)$. 
For the image classification task using neural networks, $A_T$ will depend on the neural network architecture as well as the constraints applied to the network parameters. 

Let $f_T^*$ denote the optimizer for the optimization problem \eqref{eq: obj-t}, and  $\pbm_T=Law(f_T^*(X_T))$ for the probability distribution of its output. Then the model distribution $\pbm_T$ depends on three factors: $L_T$, the conditional distribution $Law(Y_T|X_T)$, and the marginal distribution $Law(X_T)$. 
Note that in direct learning, this optimizer $f_T^*\in A_T$ is solved directly by analyzing the optimization problem \eqref{eq: obj-t}, whereas in    transfer learning, one  leverages knowledge from the source task to facilitate the search of $f_T^*$. 

\paragraph{Source task $S$.}
 In the source task $S$, denote $\XCal_S$ and $\YCal_S$ as the input and output spaces of the source task, respectively, and 
$(X_S,Y_S)$ as a pair of $\XCal_S\times\YCal_S$-valued random variables.  Here, $(\XCal_{S},\|\cdot\|_{\XCal_{S}})$ and $(\YCal_{S},\|\cdot\|_{\YCal_{S}})$ are Banach spaces with norms $\|\cdot\|_{\XCal_{S}}$ and $\|\cdot\|_{\YCal_{S}}$, respectively. Let $L_S:\YCal_S\times\YCal_S\to\R$ be a real-valued function and let us assume that the learning objective for the source task is 
\begin{equation}\label{eq: obj-s}
    \min_{f_S\in A_S}\LL_S(f_S)= \min_{f\in A_S}\E[L_S(Y_S,f_S(X_S))],
\end{equation}
where $\LL_S(f_S)$ is the loss function for a model $f_S:\XCal_S\to\YCal_S$ for the source task $S$. Here $A_S$ denotes the set of source task models such that
\begin{equation}\label{eq: a-s}
A_S\subset 
\{f_S|f_S:\XCal_S\to\YCal_S\}.
\end{equation}

Moreover, denote the optimal solution for this optimization problem \eqref{eq: obj-s} as $f_S^*$, and the probability distribution of the output of $f_S^*$ by $\pbm_S=Law(f_S^*(X_S))$.  Meanwhile, similar as the target model, the model distribution $\pbm_S$ will depend on the function $L_S$, the conditional distribution $Law(Y_S|X_S)$, and the marginal distribution $Law(X_S)$. 

Back to the image classification example,  the target task may only contain  images of items in an office environment, the source task may have more image samples from a richer dataset, e.g., ImageNet. Meanwhile, $\XCal_S$ and $\YCal_S$ may have different dimensions compared with $\XCal_T$ and $\YCal_T$, since the image resolution and the class number vary from task to task.
Similar to the admissible set $A_T$ in the target task, $A_S$ depends on the task description, and $f_S^*$ is usually a deep neural network with parameters pretrained using the source data.

In transfer learning, the optimal model $f_S^*$ for the source task is also referred to as a pretrained model. The essence of transfer learning   is to utilize this pretrained model $f_S^*$ in the source task to accomplish the optimization objective \eqref{eq: obj-t}. 
We now define this procedure in three steps.

\paragraph{Step 1. Input transport.}
Since  $\XCal_T$ is not necessarily contained by the source input space $\XCal_S$, the first step is therefore to make an appropriate adaptation to the target input $X_T\in\XCal_T$. In the  example of image classification, popular choices for input transport may include resizing, cropping, rotation, and grayscale. We define this adaptation as an {\it input transport mapping}.
\begin{defn}[Input transport mapping]
    \label{defn: inp-tr}
    A function 
    \begin{equation}\label{eq: t-x}
        T^X\in\{f_\text{input}|f_\text{input}:\XCal_T\to\XCal_S\}
    \end{equation}
    is called an input transport mapping with respect to the source and target task pair $(S,T)$ if it takes any data point in the target input space $\XCal_T$ and maps it into the source input space $\XCal_S$. 
\end{defn}
With an input transport mapping $T^X,$ the first step of transfer learning can be represented as follows.
\begin{equation*}
   \XCal_T\ni X_T\xmapsto{\text{Step 1. Input transport by }T^X}T^X(X_T)\in\XCal_S.
\end{equation*}

In a class of transfer learning called domain adaption, it is assumed that the difference between the source input distribution $Law(X_S)$ and target input distribution $Law(X_T)$ is the only factor to motivate the transfer, while the labeling function of the source and target tasks stays the same. (See also Section \ref{subsec:example} for more details on domain adaptation). Therefore, once a proper input transport mapping $T^{X}$ is found, transfer learning is accomplished. Definition \ref{defn: inp-tr} is thus consistent with \cite{flamary2016optimal}, in which domain adaption is formulated as an optimal transport from the target input to the source input. 

For most transfer learning problems,  however, one needs both a transport mapping for the input {\it and} a transport mapping for the output. For instance, the labeling function for different classes of computer vision tasks, such as object detection, instance segmentation, and image classification, can vary greatly and depend on the specific task. Hence, the following two more steps are required.

\paragraph{Step 2. Applying pretrained model.}
After applying an input transport mapping $T^X$ to the target input $X_T$, the pretrained model $f_S^*$ will take the transported data $T^X(X_T)\in\XCal_S$  as an input. That is, 
\begin{equation*}\label{eq: f_S^*}
   \XCal_S\ni T^X(X_T)\xmapsto{\text{Step 2. Apply }f_S^*}(f_S^*\circ T^X)(X_T)\in\YCal_S,
\end{equation*}
where  $(f_S^*\circ T^X)(X_T)$ denotes the corresponding output of the pretrained model $f_S^*$. Note here the composed function $f_S^*\circ T^X\in
\{f_\text{int}|f_\text{int}:\XCal_T\to\YCal_S\}$.

\paragraph{Step 3. Output transport.}
After utilizing the pretrained model $f_S^*$, the resulting model $f_S^*\circ T^X$ may, however, still be inadequate for the target model:  one may need to map the $\YCal_S$-valued output into the target output space $\YCal_T$. Hence, it is necessary to define an {\it output transport mapping}.
\begin{defn}[Output transport mapping]
    \label{defn: out-tr}
    A function
    \begin{equation}
        T^Y\in
        \{f_\text{output}|f_\text{output}:\XCal_T\times\YCal_S\to\YCal_T\}
    \end{equation}
    is called an output transport mapping with respect to the source and target task pair $(S,T)$ if, for an optimal source model $f_S^*:\XCal_S\to\YCal_S$, the composed function
    $T^Y(\cdot,f_S^*(\cdot))\in A_T.$
\end{defn}
Now, this third and the final step in transfer learning can be expressed as
\begin{align*}\label{eq: output-tr}
   \XCal_T\times\YCal_S\ni (X_T,(f_S^*\circ T^X)(X_T))\xmapsto{\text{Step 3. Output transport by }T^Y}T^Y\left(X_T,(f_S^*\circ T^X)(X_T)\right)\in\YCal_T.
\end{align*}

For the image classification task with transfer learning, the optimal source model usually consists of the first few layers of the neural network for feature extraction, and the output transport mapping is the subsequent prediction layers that map the features from the optimal source model to the target output labels. See Section \ref{subsec:example} for more details.

An output transport mapping can also be viewed as an operation to tailor the optimal source model into a suitable target model. For instance, in \cite{xia2022structured}, a large language model is a collection of optimal pretrained transformer models and each model consists of a multi-head self-attention layer and feed-forward layer. Thus, the output transport mapping is  the {\it structure pruning with distillation} operation  applied to each optimal transformer model, where pruning reduces the original transformer model to a simplified sub-model which is more suitable for the corresponding down-stream tasks, and where distillation ensures the proper knowledge is passed from the source model down to the target model.

Combining these three steps, transfer learning  can be presented by the following diagram,
\begin{equation}\label{eq: tl-fw}
    \begin{matrix}
        \XCal_S\ni X_S & \xRightarrow{\text{\hspace{4pt} Pretrained model } f_S^* \text{ from } \eqref{eq: obj-s}\text{\hspace{4pt}}} & f_S^*(X_S)\in\YCal_S\\
       T^X\Big\Uparrow & & \Big\Downarrow T^Y \\
        \XCal_T\ni X_T & \xdashrightarrow[\text{\hspace{10pt}}f_T^*\in\argmin_{f\in A_T} \LL_T(f_T)\text{\hspace{10pt}}]{\text{Direct learning \eqref{eq: obj-t} }} & f_T^*(X_T)\in\YCal_S
    \end{matrix}
\end{equation}

\subsection{Optimization Formulation and Feasibility of Transfer Learning} In summary, transfer learning aims to find an appropriate pair of input and output transport mappings $T^X$ and $T^Y$, where the input transport mapping $T^X$ translates the target input $X_T$ back to the source input space $\XCal_S$ in order to utilize the optimal source model $f_S^*$, and the output transport mapping $T^Y$ transforms a $\YCal_S$-valued model to a $\YCal_T$-valued model. 
This is in contrast to the  direct learning, where the optimal model  $ f_T^*$ is derived by solving the optimization problem in the target task \eqref{eq: obj-t}. 
In other words, transfer learning is the following optimization problem.
\begin{defn}[Transfer learning]\label{def:tl}
The three-step transfer learning procedure presented in \eqref{eq: tl-fw} is to solve the optimization problem
\begin{align}
\label{eq: doub-trans}
\min_{T^X\in\mathbb{T}^X,T^Y\in\mathbb{T}^Y}\LL_T\left(T^Y(\cdot, (f_S^*\circ T^X)(\cdot))\right)=\E\left[L_T\left(Y_T,T^Y(X_T,(f_S^*\circ T^X)(X_T))\right)\right].
\end{align}
Here, $\T^X$ and $\T^Y$ are proper sets of transport mappings such that 
\[\left\{T^Y(\cdot, (f_S^*\circ T^X)(\cdot))|T^X\in\T^X,T^Y\in\T^Y\right\}\subset A_T.\]
In particular, when $\XCal_S=\XCal_T$ (resp. $\YCal_S=\YCal_T$), the identity mapping  $id^X(x)=x$ (resp.  $id^Y(x,y)=y$) is included in $\T^X$ (resp. $\T^Y$).
\end{defn}


This optimization reformulation of the three-step transfer learning procedure provides potentially a unified framework to analyze the impact and implications of various transfer learning techniques,
including resizing, cropping, pruning, and distillation. Moreover, it enables us to analyze the feasibility of transfer learning, which we establish  in terms of the following well-definedness of the corresponding optimization problem \eqref{eq: doub-trans}.

\begin{thm}\label{thm:existence}
    Under suitable choices of loss functions for $\mathcal{L}_T$ and appropriate compactness assumptions, there exists optimal solutions for optimization problem \eqref{eq: doub-trans}.
\end{thm}

Detailed assumptions and proof for Theorem \ref{thm:existence} is deferred to Appendix \ref{app:thm-exist}.

The procedure of solving this optimization problem is often referred to as fine-tuning in the literature of transfer learning. It is to choose some initial transport mappings $T^X_0\in\T^X_0\subset\T^X$ and $T^Y_0\in\T^Y_0\subset\T^Y$ to derive an intermediate model $f_{ST}\in A_T$ with
\begin{equation}\label{eq: int-model}
    f_{ST}(x)=T^Y_0(x, (f_S^*\circ T^X_0)(x)),\quad \forall x\in\XCal_T,
\end{equation}
with the set of possible intermediate models denoted as
\begin{equation}\label{eq: int-set}
    \ICal=\left\{T^Y_0(\cdot, (f_S^*\circ T^X_0)(\cdot))\big|T^X_0\in\T^X_0, T^Y_0\in\T^Y_0\right\}.
\end{equation}

This fine-tuning procedure allows for computationally efficient evaluation of transferability in terms of {\it transfer risk}, to be introduced in Section \ref{subsec:def_risk}.

\subsection{Examples.}\label{subsec:example}
\paragraph{Image classification.} Consider a transfer learning task in image classification using the Office-31 \cite{saenko2010adapting} benchmark dataset, which consists of images from three domains: Amazon (A), Webcam (W) and DSLR (D). In total, the dataset contains 4110 images of 31 categories of objects typically found in an office environment. Samples from the Office-31 dataset are shown in Figure \ref{fig:office31_sample}.

\begin{figure}[H]
    \centering
    \includegraphics[width=0.5\textwidth]{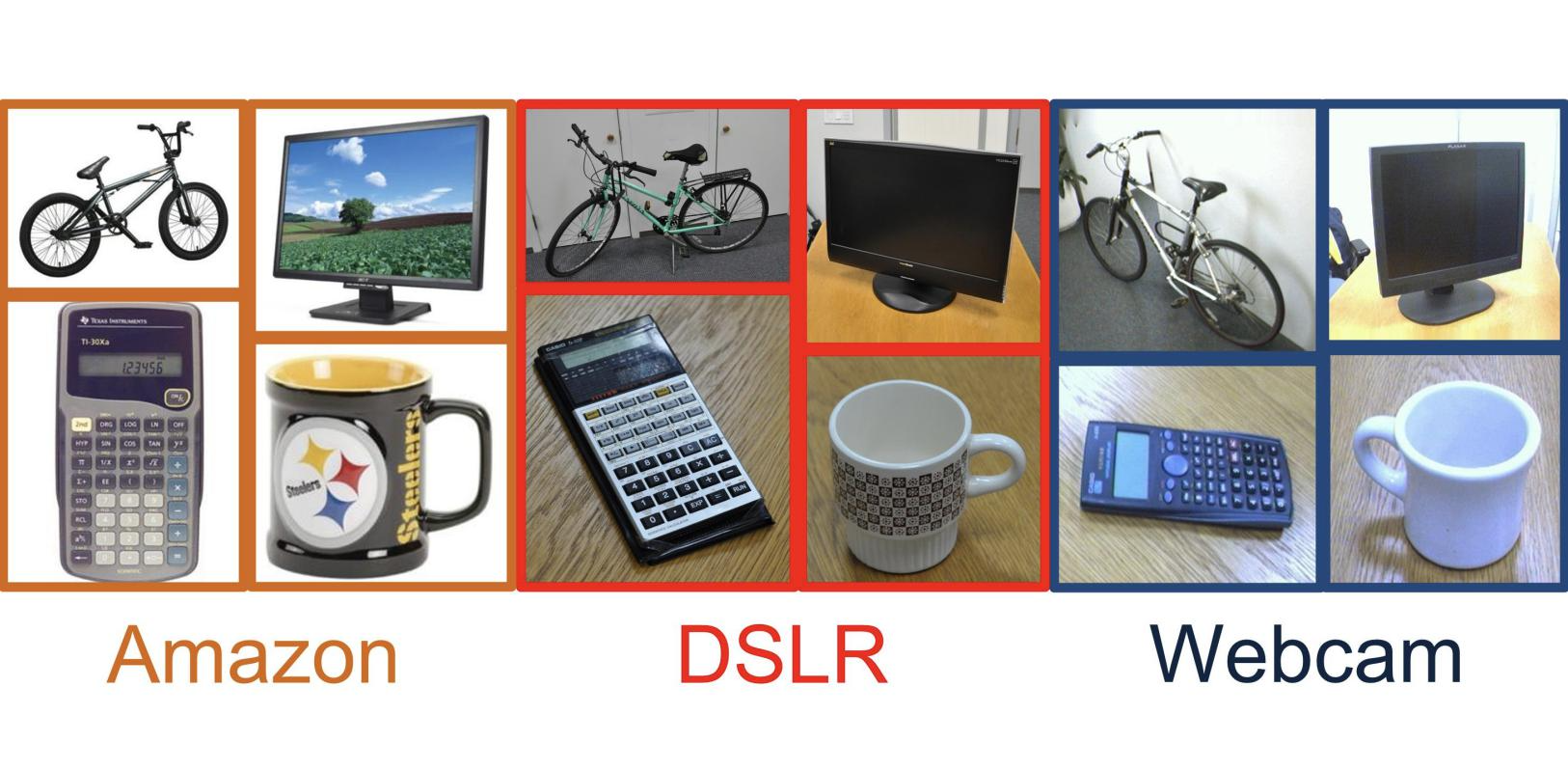}
    \caption{Samples from Office-31.}
    \label{fig:office31_sample}
\end{figure}

The neural network architecture for the image classification task is shown in Figure \ref{fig:office31_archi}. It sequentially consists of: 1) a data-preprocessing module which resizes a input image to $3\times 244\times 244$ dimension; 2) ResNet50 as a feature extractor whose output is a 2048-dimensional feature vector; and 3) a two-layer neural network which maps a 2048-dimensional feature vector to a 31-dimensional probability vector.
\begin{figure}[H]
    \centering
    \includegraphics[width=0.8\textwidth]{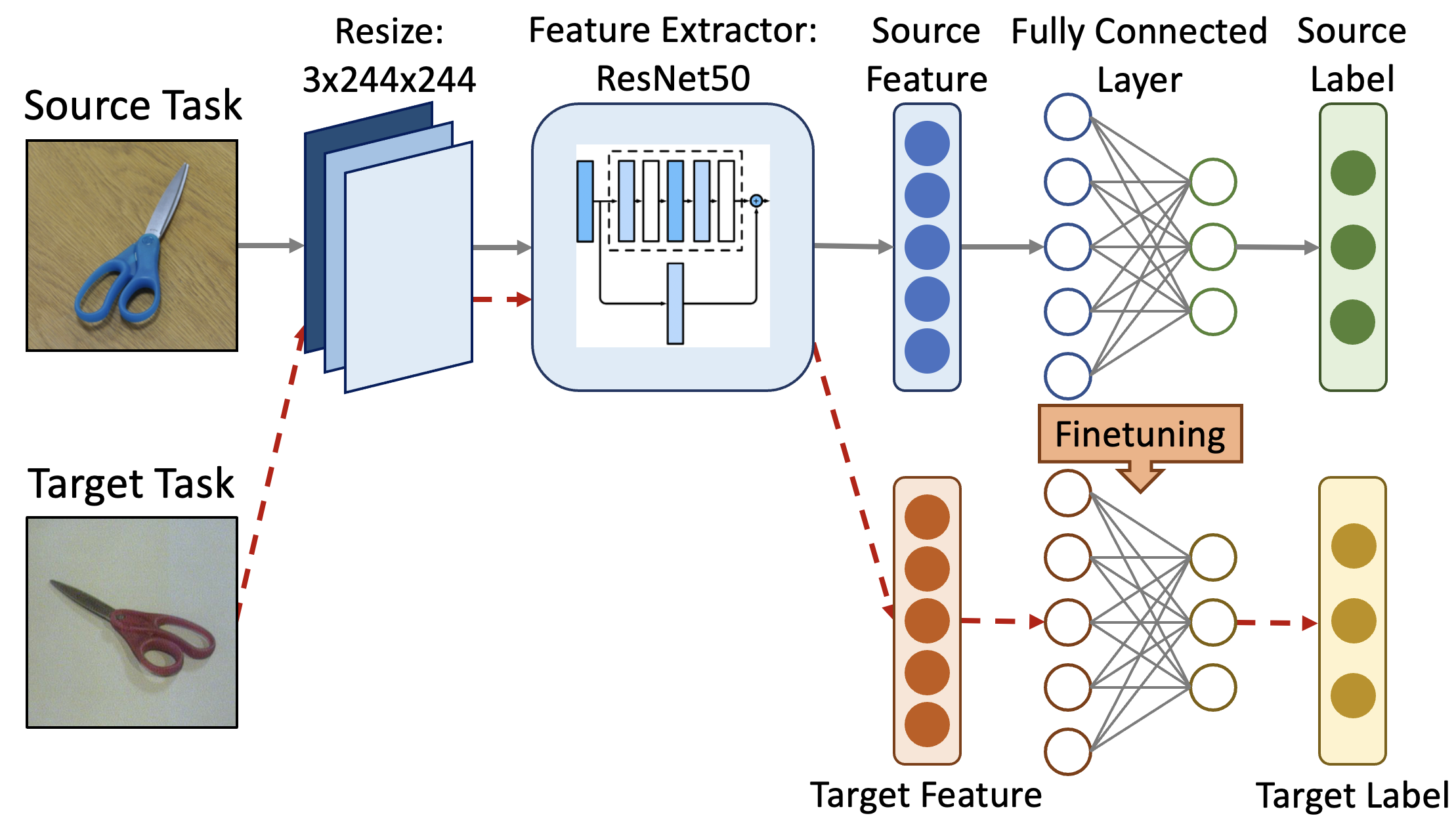}
    \caption{Neural network architecture for Office-31.}
    \label{fig:office31_archi}
\end{figure}

In this example, the source task can be chosen from any of three domains (A, D, or W), with $\mathcal{X}_S=\mathbb{R}^{3\times 244\times 244}$ being the space of resized image samples from the source domain, and 
$$\mathcal{Y}_S=\Delta_{31}:=\{p\in\mathbb{R}^{31}: \sum_1^{31} p_i=1, p_i\geq 0, \forall 1\leq i\leq 31\}$$ 
being the space of image class labels. Similarly, for any target task (A, D, or W), 
$$\mathcal{X}_T=\mathcal{X}_S=\mathbb{R}^{3\times 244\times 244}$$
is the space of resized image samples from the target domain, and $\mathcal{Y}_T=\mathcal{Y}_S=\Delta_{31}$. For both the source and the target tasks, the loss function $L_S=L_T$ is chosen to be the cross entropy between the actual label and the predicted label.

As introduced in Figure \ref{fig:office31_archi}, the set of source models are given by 
$$A_S=\{f_\text{NN}\circ f_\text{Res}:\mathcal{X}_S\to\mathcal{Y}_S| f_\text{NN}\in \text{NN}^{31}_{2048}, f_\text{Res}\in \text{Res}^{2048}_{3\times 244\times 244}\}.$$
Here $\text{Res}^{2048}_{3\times 244\times 244}$ denotes all ResNet50 architectures with $3\times 244\times 244$-dimensional input and 2048-dimensional output, and $\text{NN}^{31}_{2048}$ denotes all two-layer neural networks which map a 2048-dimensional feature vector to a 31-dimensional probability vector in $\mathcal{Y}_S$. The source model $f^*_{\text{Res},S}$ and $f^*_{\text{NN},S}$ is obtained by solving the source task optimization \eqref{eq: obj-s}.

To transfer the source model to the target task, the pretrained ResNet50 model $f^*_{\text{Res},S}$ will be fixed, while the last two-layer classifier $f_\text{NN}\in \text{NN}^{31}_{2048}$ will be fine-tuned using part of the data from the target domain $(\mathcal{X}_T, \mathcal{Y}_T)$.  The input transport set $\mathbb{T}^X$ in this example is a singleton set whose element is the identity mapping on $\mathbb{R}^{3\times 244\times 244}$. Meanwhile, the set of output transport mappings is given by 
\begin{equation}\label{eqn:nnclass}
    \mathbb{T}^Y=\{f_\text{NN}\circ f^*_{\text{Res},S}:\mathcal{X}_T\to\mathcal{Y}_T| f_\text{NN}\in \text{NN}^{31}_{2048}\}.
\end{equation}
The transfer learning task is formulated as 
$$\min_{T^Y\in\mathbb{T}^Y}\E\left[L_T\left(Y_T,T^Y(X_T)\right)\right].$$ 
Note the formulation is slightly simpler than  \eqref{eq: doub-trans} because in this particular example, the output transport in $\mathbb{T}^Y$ takes inputs from $\mathcal{X}_T$ instead of $\mathcal{X}_T\times\mathcal{Y}_S$.
Furthermore, in this example, there is no additional constraint on intermediate models defined in \eqref{eq: int-model}. Therefore, the set $\mathcal{I}$ defined in \eqref{eq: int-set} is equivalent to $\mathbb{T}^Y$ in \eqref{eqn:nnclass}.

\paragraph{Domain adaption.} This class of transfer learning problem considers the case where the output variable for the source and target tasks coincides, i.e., $Y_S=Y_T=Y\in\YCal$, and there exists some one-to-one input transport $T^X$ such that $T^X(X_T)=X_S$ almost surely \cite{flamary2016optimal}. Here we define the family of admissible (initial) output transport mappings as $\T^Y_0=\T^Y=\{id_{\YCal}\}$, where $id_{\YCal}$ denotes the identity mapping on $\YCal$; and define the family of admissible (initial) input transport mappings as $\T^X_0=\T^X=\{T^X:\XCal_T\to\XCal_S\,|\,T\text{ is one-to-one}\}$. Then $\ICal=\{f^*_S\circ T|T\in\T^X_0\}$. 
When the loss functions for the source and the target tasks are also in the same form such that $L_S=L_T=L:\YCal\times\YCal\to\R$, it can be shown that the optimal source model and optimal target model satisfy the relation $f^*_T=f^*_S\circ T^X$, where
\begin{align*}
f^*_\cdot:=\argmin_{f:\XCal_\cdot\to\YCal}\mathbb{E}[L(Y, f(X_\cdot))].
\end{align*}
From the transfer learning perspective, $T^X$ is also the optimal solution to the optimization problem \eqref{eq: doub-trans}. In particular, the transfer learning model $f^*_T=f^*_S\circ T^X$ is equivalent to the optimal model from the direct learning, while solving the transfer learning problem \eqref{eq: doub-trans} may require much less data.

\section{Transfer Risk and Transferability of Transfer Learning}
Given  the mathematical framework and after the feasibility analysis of transfer learning, we will now propose  a novel notion of {\it transfer risk}, to analyze the effectiveness and the appropriateness of transfer learning over the set of all intermediate models $\ICal$ given by \eqref{eq: int-set}.
 
\subsection{Transfer Risk}\label{subsec:def_risk}
The idea is to re-interpret the transfer learning framework \eqref{eq: tl-fw} in a sequential manner:  the mapping $T^X$ first transports $Law(X_T)$ to some probability distribution $Law(T^X(X_T))$; then, applying the pretrained model $f^*_S$ for the optimization problem \eqref{eq: obj-s} yields the distribution $\tilde\pbm_S=Law(f^*_S(T^X(X_T)))$.
Finally, an output transport mapping $T^Y$, together with  the target input $X_T$,  transports the distribution $\tilde\pbm_S$ to $\pbm_T$. 
That is,  the transfer learning scheme can be viewed as the composition of the following two steps.
\begin{enumerate}
    \item (Psuedo) Domain adaption, which can also be seen as optimal transport from $Law(X_T)$ to $Law(X_S)$.
    \item Optimal transport from $\tilde\pbm_S=Law(f^*_S(T^X(X_T)))$ to $\pbm_T$ over $T^Y$. 
\end{enumerate}

In other words, in parallel to the three-step procedure in transfer learning, there are two major sources of transfer risk for a fixed intermediate model $f_{ST}\in\ICal$:  the risk that measures the mismatch  between the output distributions of the intermediate model $f_{ST}$ and the optimal target model $f_T^*$, and the risk reflecting the difference between the transported target input and the source input. 

Let us first define the risk associated  the output transport mapping.
\begin{defn}[Output transport risk]\label{defn:outputrisk}
Let \(\Ef^O:A_T\to\R\) be a real-valued function on the set of target models. For any \(f_{ST}\in\ICal\subset A_T\), $\Ef^O(f_{ST})$ is called an output transport risk of intermediate model $f_{ST}$ if it  satisfies
    \begin{enumerate}
        \item $\Ef^O(f_{ST})\geq 0$, i.e., transfer learning always incurs a non-negative effort;
        \item $\Ef^O(f_{ST})=0$ if and only if $\pbm_T=\pbm_{ST}$, where $\pbm_T:=Law(f_{T}(X_T))$ and $\pbm_{ST}:=Law(f_{ST}(X_T))$. That is, the output transport risk vanishes when the intermediate model $f_{ST}$ completely recovers the distribution of the optimal target task.
    \end{enumerate}
\end{defn}

Clearly, the smaller this output risk, the more effective the transfer scheme with the intermediate model $f_{ST}$. 

We next define the risk associated with the input transfer.
\begin{defn}[Input transfer risk]\label{defn: inputrisk}
Let \(\Ef^I:\T^X\to\R\) be a real-valued function on the set of input transport mappings. Given an import transport mapping $T^X_0\in\T^X_0\subset \T^X$, $\Ef^I(T^X_0)$ is called an input transport risk if it  satisfies
    \begin{enumerate}
        \item $\Ef^I(T^X_0)\geq 0$, i.e., transfer learning always incurs a non-negative effort;
        \item $\Ef^I(T^X_0)=0$ if and only if $T^X_0\#Law(X_T)=Law(X_S)$. 
    \end{enumerate}
\end{defn}
The smaller this input risk, the higher the similarity between the transported target input $T^X_0(X_T)$ and the source input $X_S$.



Note  that these definitions of risks involve the sets of initial transport mappings $\T^X_0$ and $\T_0^Y$, instead of the sets of all possible transport mappings $\T^X$ and $\T^Y$. These reduced sets allow for efficient evaluation of transfer risk prior to starting the full-scale transfer learning.

Both the input transfer risk and the output transfer risk are functions characterizing the divergence between probability distributions, and their exact forms can be task dependent.  Nevertheless,  there is a key difference between these two forms of risks:  in the output transport risk, $\pbm_T$, the output distribution of the optimal target model, is {\it unknown}, and no prior knowledge about $f_T^*$ is assumed. 
Therefore, analyzing the output transport risk is decisively more complicated.
See more detailed discussions in Section \ref{sec:divergence}.

We are now ready to propose the notion of  {\it transfer risk}  by considering all intermediate models in $\ICal$, in order to  measure the effectiveness of a transfer learning framework \eqref{eq: doub-trans}.
\begin{defn}[Transfer risk]\label{defn: trans-bene}
For a transfer learning procedure characterized by the 6-tuple ${(S,T,\T^X,\T^X_0,\T^Y,\T^Y_0)}$  in \eqref{eq: doub-trans}, the transfer risk of the transfer learning framework \eqref{eq: doub-trans} from source task $S$ to target task $T$ is defined as
\begin{equation}\label{eqn:risk}
\CCal(S,T)=\inf_{f_{ST}\in\ICal}\CCal(S,T|f_{ST}).
\end{equation}
Here, for a given $f_{ST}=T^Y_0(\cdot, (f_S^*\circ T^X_0)(\cdot))\in\ICal$, $\CCal(S,T|f_{ST})$ is called  model-specific transfer risk such that $\CCal(S,T|f_{ST})\ge 0$ with the following properties:
\begin{enumerate}
    \item Let $C:\R\times\R\to\R$ with $C(0,0)=0$. $\CCal(S,T|f_{ST})=C(\Ef^O(f_{ST}),\Ef^I(T^X_0))$ is non-decreasing in $\Ef^O(f_{ST})$ under any fixed $\Ef^I(T^X_0)$ and non-decreasing in $\Ef^I(T^X_0)$ under any fixed $\Ef(f_{ST})$; 
    \item $\CCal(S,T|f_{ST})$ is Lipschitz  in the sense that for any other transfer problem characterized by $(\bar S,\bar T,\bar\T^X,\bar\T_0^X,\bar\T^Y,\bar\T_0^Y)$ and one of its intermediate models $\bar f_{ST}=\bar T^Y_0(\cdot, (\bar f_S^*\circ \bar T^X_0)(\cdot))\in\bar\ICal$, there exists a constant $L>0$ such that 
    \[\begin{aligned}|\CCal(S,T|f_{ST})-\CCal(\bar S,\bar T|\bar f_{ST})|&\leq L(|\Ef^O(f_{ST})-\Ef^O(\bar f_{ST})|\\
    &+|\Ef^I(T^X_0)-\Ef^I(\bar T^X_0)|).\end{aligned}\]
\end{enumerate}
\end{defn}
The expression of this Lipschitz property in Definition \ref{defn: trans-bene} is to emphasize the dependence of transfer risk on a given transfer learning problem. This Lipschitz property is satisfied when the function $C$ in Definition \ref{defn: trans-bene} is Lipschitz continuous. 

One simple example of the model-specific transfer risk is
\begin{equation}\label{eq: risk-ln}
    \CCal^\lambda(S,T|f_{ST})=\Ef^O(f_{ST})+\lambda \Ef^I(T^X_0),
\end{equation}
where $\lambda>0$ is a pre-specified parameter modulating the weight of the input transport in the transfer learning problem \eqref{eq: tl-fw}. 

Transfer risk in Definition \ref{defn: trans-bene} unifies the analysis of the risk from both the input and the output transport mappings. It allows for studying  the trade-off between them. Moreover, two of its key components, the input and the output transfer risks in Definitions \ref{defn: inputrisk} and \ref{defn:outputrisk} generalize earlier works on transferability. For instance, the $\mathcal{H}$-score proposed in \cite{bao2019information} addresses transferability of a particular classification setting and can be incorporated into the output transfer risk in Definition \ref{defn:outputrisk}. Earlier works on the relation between source and target inputs such as \cite{saenko2010adapting, ganin2016domain, long2014transfer} correspond to the special case in Definition \ref{defn: inputrisk} with $T^X_0$ being the identity mapping. 

Furthermore, one can establish the following properties of transfer risk: a)  there is zero transfer risk if the source and the target tasks are identical; and b) transfer risk is continuous in the input distribution and robust with respect to the  pretrained model. (See the exact mathematical statement and analysis of these properties in Appendix \ref{subsec: property}). The continuity of the transfer risk in terms of the changes in the input and the pretrained model is useful to exclude {\it a priori} inappropriate source tasks when compared against existing viable source tasks.

\subsection{Examples}\label{subsec:example_revisited}
We now revisit some examples in Section \ref{subsec:example} and their associated transfer risks based on Definition \ref{defn: trans-bene}. In particular, we will illustrate how the two key components of the transfer risk, namely, the input transport risk $\Ef^I$ and the output transport risk $\Ef^O$, are embedded in transfer learning for a given intermediate model $f_{ST}$.

\paragraph{Transfer risk in domain adaption.}\label{subsec: DA}
Recall the domain adaptation problem in Section \ref{subsec:example}, and consider the case where the transfer risk is independent of the output transport risk, i.e., the input risk $\Ef^I(T^X_0)$ completely determine the transfer risk:
\[\CCal(S,T|f_{ST})=\Ef^I(T^X_0).\] 
In this case, there exists a one-to-one input mapping $T^X\in\T^X_0$ such that $T^X(X_T)=X_S$ almost surely, implying $T^X\#Law(X_T)=Law(X_S)$ and consequently $\CCal(S,T)=\CCal(S,T|f^*_S\circ T^X)=\Ef^I(T^X)=0$. Therefore, vanishing input transport risk is a necessary condition for the domain adaptation framework to hold. Thus, the input transport risk may be adopted to check the viability of domain adaptation on certain tasks.

\paragraph{Transfer risk in image classification.}
Recall the image classification problem introduced in Section \ref{subsec:example}. Fix a source task $S$ and a target task $T$. Since the input transport set $\mathbb{T}^X$ in this problem is a singleton set, the input transport risk $\Ef^I$ is a constant depending on $Law(X_S)$ and $Law(X_T)$, with the output transport risk  denoted as $\Ef^O(f_{ST})$ for any $f_{ST}\in\mathcal{I}$ in \eqref{eqn:nnclass}. By Definition \ref{defn: trans-bene}, the model-specific transfer risk $\mathcal{C}(S,T|f_{ST})=C(\Ef^I, \Ef^O(f_{ST}))$ for some appropriate function $C:\R^2\to\R$ satisfying conditions stated in Definition \ref{defn: trans-bene}. In particular, since the function $C$ is non-decreasing with respect to $\Ef^O(f_{ST})$, minimizing $\mathcal{C}(S,T|f_{ST})$ over $f_{ST}\in\mathcal{I}$ is equivalent to minimizing $\Ef^O(f_{ST})$ over $f_{ST}\in\mathcal{I}$:
\begin{equation*}
    \argmin_{f_{ST}\in\mathcal{I}}\mathcal{C}(S,T|f_{ST})=\argmin_{f_{ST}\in\mathcal{I}}\Ef^O(f_{ST}).
\end{equation*}
And consequently,
\begin{equation*}
    \mathcal{C}(S,T)=\min_{f_{ST}\in\mathcal{I}}\mathcal{C}(S,T|f_{ST})=C(\Ef^I,\min_{f_{ST}\in\mathcal{I}}\Ef^O(f_{ST})).
\end{equation*}

\subsection{Transfer Risk and Choices of Divergence Functions}
\label{sec:divergence}
Clearly, different learning tasks may require  different choices of divergence functions for assessment of transfer risk. In this section, we present two forms of transfer risks
based on two divergence functions, and analyze their properties and relations.

\paragraph{KL-based output transport risk.} For learning tasks such as the classification problem, one may use  cross-entropy as the loss function. 

Specifically, let $\pbm_T=\tilde\pbm_T+\pbm_0$ be its unique Lebesgue decomposition, i.e., for any measurable set $B\subset\YCal_T$, there exists some  function $h_{ST}:\YCal_T\to\R^+$ such that
$\tilde\pbm_T(B)=\int_Bh_{ST}d\pbm_{ST}$, with $\pbm_0$ singular with respect to $\pbm_{ST}$. Then the KL-based output risk can be  defined as
    \[
    \Ef^O_{KL}(f_{ST}):=D_{KL}(\tilde\pbm_T\|\pbm_{ST})+{H}(\pbm_0),
    \]
where $H(\pbm_0)$ is the entropy function of $\pbm_0$.

\begin{prop}\label{lem: kl-heur}
    For a classification problem over $K\in\mathbb{N}$ classes with cross entropy as the training loss,  for any $f_{ST}\in\mathcal{I}$, 
    \begin{eqnarray*}
    \sum_{i=1}^K\log p_{ST}(i) \leq H(\pbm_T,\pbm_{ST})-H(Law(Y_T),\pbm_{ST}) \leq -\sum_{i=1}^K\log p_{ST}(i),
    \end{eqnarray*}
    where $p_{ST}$ denotes the probability mass function for $\pbm_{ST}$.
\end{prop}
Note that $H(Law(Y_T),\pbm_{ST})$ is indeed the cross-entropy loss for the classifier $f_{ST}$. Therefore, in actual training, one may use $H(Law(Y_T),\pbm_{ST})\pm\sum_{i=1}^K\log p_{ST}(i)$ to replace   $\Ef^O_{KL}(f_{ST})$.

\paragraph{Wasserstein-based output transport risk.}
For learning problems such as GANs or supervised learning with domain adaption, Wasserstein and related distances are popular choices to measure the distance between the generative distribution and the target distribution. Therefore, a Wassertein-based output risk is a natural choice  related to such learning targets. 

More specifically, for $p\geq1$, let $\mathcal{P}_p(\YCal_T)$ be the set of probability measures over $\YCal_T$ such that
\[\int_{\R^{d_{O,T}}}\|x\|_{\YCal_T}^pd\mu(x)<\infty,\quad\forall\mu\in\mathcal{P}_p(\YCal_T).\]
The Wasserstein-based output risk is defined as
\begin{align}\label{eqn:W_risk}
\Ef^O_{W}(f_{ST}):=\mathcal{W}_p(\pbm_{ST},\pbm_{T})^p:=\inf_{\gamma\in\Pi(\pbm_{ST},\pbm_{T})}\int_{\R^{d_{O,T}}\times\R^{d_{O,T}}}\|x-y\|_{\YCal_T}^pd\gamma(dx,dy),
\end{align}
for some suitable choice of $p\geq1$, where $\Pi(\pbm_{ST},\pbm_{T})$ denotes the set of couplings of probability measures $\pbm_{ST}$ and $\pbm_{T}$.

Analogy to Proposition \ref{lem: kl-heur} is the following property for $\Ef^O_W(f_{ST})$, based on the triangle inequality of the Wasserstein distance.
\begin{prop}\label{lem: w-heur}
    The Wasserstein-based output risk $\Ef^O_W$ in (\ref{eqn:W_risk}) is upper bounded in the following sense:
    \[\Ef^O_W(f_{ST})\leq 2^{p-1}[\mathcal{W}_p(\pbm_{ST},Law(Y_T))^p+\mathcal{W}_p(\pbm_{T},Law(Y_T))^p].\]
\end{prop}
 
Now, consider any intermediate model $f_{ST}$, then Talagrand's inequality 
\cite{talagrand1996transportation} gives 
$$\Ef^I_W(T_0^X)\leq 2\Ef^I_{KL}(T_0^Y),\Ef^O_W(f_{ST})\leq 2\Ef^O_{KL}(f_{ST}).$$
In particular, the linear transfer risk defined in \eqref{eq: risk-ln} satisfies
\begin{align}\label{eqn:risk-ln-KLW}
    \CCal^\lambda_{W}(S,T|f_{ST}):=\Ef^O_W(f_{ST})+\lambda\cdot\Ef^I_W(T_0^X)\leq 2\CCal^\lambda_{KL}(S,T|f_{ST}):=2(\Ef^O_{KL}(f_{ST})+\lambda\cdot\Ef^I_{KL}(T_0^X))
\end{align}
Such a relation between KL- and Wasserstein-based linear transfer risks \eqref{eqn:risk-ln-KLW} gives the following proposition.
\begin{prop}\label{prop: talagrand}
Consider transfer risk in linear form as in \eqref{eqn:risk-ln-KLW}. Suppose $\YCal_T$ is a finite-dimensional Euclidean space and $\pbm_T\ll\pbm_{ST}$. Then for a given transfer learning problem $(S,T,\T_X,\T_Y,\T_X^0,\T_Y^0)$, 
\[\CCal_W(S,T)\leq 2\CCal_{KL}(S,T).\]
\end{prop}

\subsection{Transfer Risk and Regret}\label{subsec:gaussian}
We will establish the connection between the transfer risk \eqref{eqn:risk} and the transfer learning performance through a linear regression example.

 Consider a source task $S$ and a target task $T$ with the same input space $\XCal_S=\XCal_T=\R^d$ and the same input space $\YCal_S=\YCal_T=\R$. Both source and target data satisfy two $(d+1)$-dimensional Gaussian distributions: $(X_\cdot,Y_\cdot)\sim \mathcal{N}(\mu_\cdot,\Sigma_\cdot)$ with
\begin{equation}\label{eq: source-data}
\mu_\cdot=\begin{pmatrix}\mu_{\cdot,X}\\\mu_{\cdot,Y}\end{pmatrix},\quad \Sigma_\cdot=\begin{pmatrix}\Sigma_{\cdot,X}&\Sigma_{\cdot,XY}\\\Sigma_{\cdot,YX}&\Sigma_{\cdot,Y}\end{pmatrix},
\end{equation}
where
$\mu_{\cdot,Y}\text{ and }\Sigma_{\cdot,Y}\in\R$, $\mu_{\cdot,X}\text{ and }\Sigma_{\cdot,XY}\in\R^d$, $\Sigma_{\cdot,YX}=\Sigma_{\cdot,XY}^\top$, and $\Sigma_{\cdot,X}\in\R^{d\times d}$. Define the sets of admissible source and target models $A_S=A_T=\{f:\R^d\to\R\}$. For any $f\in A_S=A_T$, define the loss function as
\begin{equation}\label{eq: lrloss}
    \LL_S(f)=\E\|Y_S-f(X_S)\|_2^2, \LL_T(f)=\E\|Y_T-f(X_T)\|_2^2.
\end{equation}

Under such a setting, the optimal source and target models are obtained by direct computations: $f_\cdot^*(x)=w_\cdot^\top x+b_\cdot$ with
\begin{align}\label{eq: source-opt}
    w_\cdot=\Sigma_{\cdot,X}^{-1}\Sigma_{\cdot,XY},\quad b_\cdot=\mu_{\cdot,Y}-\Sigma_{\cdot,YX}\Sigma_{\cdot,X}^{-1}\mu_{\cdot,X}.
\end{align}

\paragraph{Transfer learning.} Take the above linear regression example, and consider a simple setting where the input (resp. output) transport set $\mathbb{T}^X$ (resp. $\mathbb{T}^Y$) is a singleton set only containing the identical mapping on $\R^d$ (resp. $\R$). Then, the transfer learning scheme \eqref{eq: doub-trans} is equivalent to directly applying the optimal source model $f^*_S$ to the target task. Consequently, the intermediate model set $\mathcal{I}$ in \eqref{eq: int-set} is also a singleton set with $\mathcal{I}=\{f^*_S\}$. 

Now, define the transfer risk in this linear regression problem as the Wassersteinn-based output transport risk as in \eqref{eqn:W_risk}:
\begin{equation}\label{eqn:linear_risk}
\CCal_W(S,T)=\CCal_W(S,T|f^*_S)=\Ef^O_W(f^*_S).
\end{equation}
\paragraph{Regret.} Next, 
define the notion of \textit{regret} as the gap between the transfer learning  and the direct learning:
\begin{equation}\label{eqn:regret}
    \RCal(S,T):=\LL_T(f^*_S)-\LL_T(f^*_T)
\end{equation}
Then, the following proposition shows that the transfer risk serves as a lower bound of the regret.


\begin{prop}\label{prop: lb}
    For transfer learning in linear regression with Gaussian data, the regret with respect to the chosen intermediate model $\RCal(S,T)$ in \eqref{eqn:regret} is lower bounded by the Wasserstein-based transfer risk in \eqref{eqn:linear_risk},
    \[\CCal_W(S,T)\leq \RCal(S,T).\]
\end{prop}
Proposition \ref{prop: lb} suggests that in evaluating the transfer learning scheme \eqref{eq: doub-trans}, transfer risk provides a proper initial indication of its effectiveness, especially for eliminating unsuitable candidate pretrained models or source tasks if the transfer risk is large. The proof of Proposition \ref{prop: lb}, together with detailed analysis of transfer risk and regret with Gaussian data, is in Appendix \ref{app:gaussian}.

\section{Numerical Experiments with Office-31}\label{sec:image_experiment}
In this section, we will demonstrate the correlation between the performance of the transfer learning scheme \eqref{eq: doub-trans} and the transfer risk \eqref{defn: trans-bene}, through numerical experimentation using the Office-31 dataset for image classification.  

\subsection{Experiment Set-up}
Recall the neural network architecture for the experiment introduced in Section \ref{subsec:example}. For each pair of the source and the target tasks, the source model is first trained using the source data, and then the fully connected layer of the pretrained model is fine tuned using half of the target data. The performance of the model is measured by the classification accuracy using the remaining of the target data.

\paragraph{Transfer risk.} Now let us define the explicit form of transfer risk for this example. Fix a source-target pair $(S,T)$. Recall that the input transport risk $\Ef^I$ is a constant since the input transport set $\mathbb{T}^X$ is a singleton set. More specifically, we define the input transport risk as 
\begin{equation}\label{eqn:office_input}
    \Ef^I:=\mathcal{W}_1(Law(X_S), Law(X_T)),
\end{equation}
which is the Wasserstein-1 distance between the empirical distribution of (resized) source images $Law(X_S)$ and the empirical distribution of (resized) target images $Law(X_T)$. Meanwhile, for any $f_{ST}\in\mathcal{I}$ \eqref{eqn:nnclass}, we define the output transport risk as $\Ef^O(f_{ST})=\mathcal{W}_1({\mathbb{P}}_{ST}, {\mathbb{P}}_T).$
Furthermore, as discussed in Section \ref{subsec:example_revisited}, the transfer risk is given by 
\begin{equation}\label{eqn:office_risk}
    \mathcal{C}(S,T)=C(\Ef^I,\min_{f_{ST}\in\mathcal{I}}\Ef^O(f_{ST})),
\end{equation}
for some function $C:\R^2\to\R$ satisfying the regularity conditions in Definition \ref{defn: trans-bene}. Note the the optimal target distribution $\mathbb{P}_T$ in the definition of $\Ef^O(f_{ST})$ is unknown {\it a priori}. Thus, as suggested by Proposition \ref{lem: w-heur}, we approximate $\Ef^O(f_{ST})$ by $\mathcal{W}_1({\mathbb{P}}_{ST}, Law(Y_T))$. Denote the approximated output transfer risk as 
\begin{equation}\label{eqn:office_out_risk}
    \widehat{\Ef}^O=\min_{f_{ST}\in\mathcal{I}}\mathcal{W}_1({\mathbb{P}}_{ST}, Law(Y_T)).
\end{equation}
Finding $\widehat{\Ef}^O$ \eqref{eqn:office_out_risk} is an optimization problem over a neural network function class $f_{ST}\in\mathcal{I}$ \eqref{eqn:nnclass}, which is solved by gradient descent in the numerical experiment. Finally, the (approximated) transfer risk is obtained by plugging $\widehat{\Ef}^O$ into \eqref{eqn:office_risk}.

\subsection{Numerical Result}
Three different domains in Office-31 (A, D, and W) lead to $3\times 2=6$ source-target pairs in total. The accuracy, the input transport risk $\Ef^I$ \eqref{eqn:office_input}, and the output transport risk $\widehat\Ef^O$ \eqref{eqn:office_out_risk} for each pair of source and target tasks are reported in the first three rows of Table \ref{tab:office31}. Here the input transport risk is rescaled by a constant factor to achieve the same scale as the other metrics.

In order to compute the transfer risk $\mathcal{C}$ in \eqref{eqn:office_risk} given $\Ef^I$ in \eqref{eqn:office_input} and $\widehat\Ef^O$ in \eqref{eqn:office_out_risk}, an appropriate form of function $C$ in \eqref{eqn:office_risk} need to be determined. In this experiment, we search $C$ from the class of second order polynomials, so as to maximize the (absolute value of) correlation between the transfer learning accuracy and the transfer risk. In particular, we define the risk in the following form:
\begin{equation}\label{eqn:office_risk_exact}
    \mathcal{C}(S,T)=0.31\cdot\Ef^I+0.92\cdot\left(\widehat\Ef^O\right)^2.
\end{equation}
Transfer risks for all source-target pair are reported in the last row of Table \ref{tab:office31}.

\begin{table}[H]
  \centering
    \begin{tabular}{c|cccccc}
    \toprule
    \toprule
    \textbf{Metric\textbackslash{}Task} & \textbf{A-W}   & \textbf{A-D}   & \textbf{W-A}   & \textbf{W-D}   & \textbf{D-A}   & \textbf{D-W} \\
    \midrule
    \textbf{Accuracy} & 80.9\% & 83.1\% & 66.9\% & 94.5\% & 66.6\% & 87.8\% \\
    \midrule
    \textbf{Input Risk} & 0.181 & 0.263 & 0.181 & 0.148 & 0.263 & 0.148 \\
    \midrule
    \textbf{Output Risk} & 0.428 & 0.380 & 0.545 & 0.084 & 0.543 & 0.412 \\
    \midrule
    \textbf{Transfer Risk} & 0.224 & 0.214 & 0.330 & 0.052 & 0.353 & 0.201 \\
    \bottomrule
    \bottomrule
    \end{tabular}
  \caption{Accuracy and transfer risk.}
  \label{tab:office31}
\end{table}

\paragraph{Accuracy v.s. transfer risk.} Figure \ref{fig:office31} demonstrates a significant negative correlation between the transfer learning accuracy and the transfer risk: the higher the risk, the lower the transfer learning accuracy. For example, it can be observed from Figure \ref{fig:office31} that transfer learning between DSLR and Webcam (D-W or W-D) results in low risk and high accuracy; while transfer learning from those domains to Amazon (D-A or W-A) is risky and suffers from low accuracy. Those numerical findings demonstrate the potential of transfer risk as an informative metric for the effectiveness of transfer learning task.

\begin{figure}[H]
    \centering
    \includegraphics[width=0.6\textwidth]{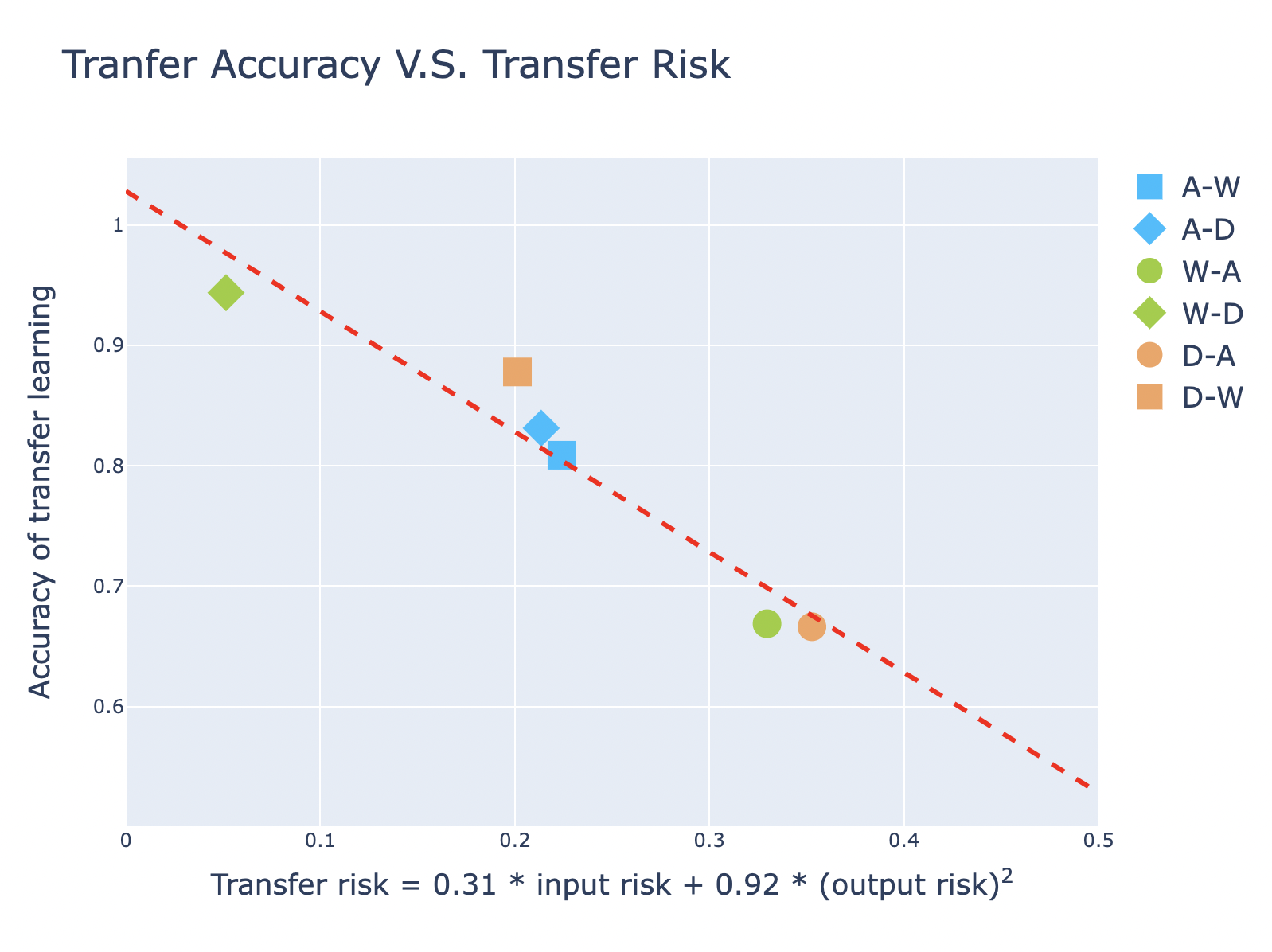}
    \caption{Accuracy and transfer risk.}
    \label{fig:office31}
\end{figure}

\paragraph{Computational benefit of  transfer risk.}  

In this numerical experiment on the Office-31 dataset, assessing transfer risk is computationally efficient and guaranteed by the early-stopping trick in deep learning: for each source-target pair, the optimization problem \eqref{eqn:office_out_risk} is solved by running the gradient descent for a small and fixed number ($\sim$10) of epochs, while the transfer learning problem is solved until the accuracy converges, which may take up to  100 epochs. This early stopping trick is essentially equivalent to shrinking the search space of the output mapping from  $\mathbb{T}^Y$ in \eqref{eqn:nnclass} to some smaller class of neural networks $\mathbb{T}_0^Y\subset\mathbb{T}^Y$. 

Indeed, as emphasized in Section \ref{subsec:def_risk}, computing transfer risk \eqref{eqn:risk} is to solve an optimization problem  over the sets $\mathbb{T}^X_0$ and $\mathbb{T}^Y_0$, which can be much smaller than the function classes $\mathbb{T}^X$ and $\mathbb{T}^Y$ involved in the transfer learning problem \eqref{eq: doub-trans}. This reduction of the function classes demonstrates the potential and benefit of adopting transfer risk for  computational efficiency:  one can first perform the much easier computing task of the transfer risk, and then assess whether or not to resort to the full-scale and more computationally intense form of transfer learning.

\section{Conclusion}
This paper establishes a mathematical framework for transfer learning, and addresses  issues of feasibility and transferability through rigorous
and comprehensive mathematical analysis.  A novel concept of transfer risk is introduced, which not only generalizes existing notions for transferability but also
provides a unified framework for future studies on the impact and implications of various transfer learning techniques, including resizing, cropping, pruning, and distillation.

\newpage

\bibliographystyle{apalike}
\bibliography{refs}

\newpage

\begin{center}
{\Large \bf {\centering Appendix}}
\end{center}

\appendix

\section{Mathematical Proofs}

\subsection{Proof of Theorem \ref{thm:existence}}\label{app:thm-exist}
We will show that the optimization problem \eqref{eq: doub-trans} is well-defined in the sense that an optimal pair of transport mappings $(T^{X,*},T^{Y,*})$ for \eqref{eq: doub-trans} is obtainable, under certain regularity conditions. More specifically, we will focus on the following type of loss function $\LL_T$. 

\begin{defn}[Proper loss function]
   Let $(X,Y)$ be a pair of $\XCal_T\times\YCal_T$-valued random variables with  $Law(X_T,Y_T)\in\mathcal{P}(\XCal_T\times\YCal_T)$. A loss functional $\mathcal{L}_T$ over $A_T$ is said to be {\em proper} with respect to $(X,Y)$ if there exist a corresponding function $L_T:\YCal_T\times\YCal_T\to\mathbb{R}$ bounded from below such that for any $f\in A_T$,
    \[\LL_T(f)=\E[L_T(Y,f(X))]=\E[\E[L_T(Y,f(X))|X]];\]
    moreover, the function $\tilde L_T:\YCal_T\to\mathbb{R}$ given by
    \[\tilde L_T(y)=\E[L_T(Y,Y')|Y'=y],\quad\forall y\in\YCal_T,\]
    is continuous. 
\end{defn}
Examples of proper loss functions include mean squared error and KL-divergence and more generally the Bregman divergence, assuming that the first and second moments of $Y$ conditioned on $Y'=y$ is continuous with respect to $y$. 
 
Without loss of generality,  we shall in this section assume the input transport set $\T^X$ contains all functions from $\XCal_T$ to $\XCal_S$. We then specify the following assumptions for the well-definedness of \eqref{eq: doub-trans}.
\begin{asp}\label{asp: A}
Assume the following regularity conditions hold.
\begin{enumerate}
    \item $\LL_T$ is a proper loss functional with respect to $(X_T,Y_T)$;
    \item the image $f_S^*(\XCal_S)$ is compact in $(\YCal_S,\|\cdot\|_{\YCal_S})$;
    \item the set $\T^Y$ is such that the following set of functions
    \[\tilde\T^Y=\{\tilde T^Y:\XCal_T\to\YCal_T\,|\,\exists T^Y\in\T^Y\text{ s.t. }\tilde T^Y(x)=\inf_{y\in f_S^*(\XCal_S)}\tilde L_T(T^Y(x,y)),\ \ \forall x\in\XCal_T\}\]
    is compact in $(\{f|f:{\XCal_T}\to\YCal_T\},\|\cdot\|_{\infty})$, where for any $f:\XCal_T\to\YCal_T$, $\|f\|_{\infty}:=\sup_{x\in\XCal_T}\|f(x)\|_{\YCal_T}$.
\end{enumerate}
\end{asp}
The proper choice of loss functions for $\mathcal{L}_T$ is fairly general and  includes the mean squared error, the KL-divergence, and more generally the Bregman divergence;  
the compactness assumptions can be fairly flexible as long as  the target optimal model $f_T^*$ can be written as
$f_T^*(x)=T^Y(x,f_S^*(T^X(x))),\quad\forall x\in\XCal_T.$
This compactness condition can be implemented by choosing a particular family of activation functions or imposing boundaries restrictions to weights and biases when constructing machine learning models. 

Now we are ready to prove Theorem \ref{thm:existence} under Assumption \ref{asp: A}.
\begin{proof}[Proof of Theorem \ref{thm:existence}]
Since $\LL_T$ is proper, there exists a function $L_T:\YCal_T\times\YCal_T\to\mathbb{R}$ such that 
\[\inf_{(y,y')\in\YCal_T\times\YCal_T}L_T(y,y')>-\infty,\]
and
\[\LL_T(T^Y(\cdot,(f_S^*\circ T^X)(\cdot)))=\E[L_T(Y_T,T^Y(X_T,(f_S^*\circ T^X)(X_T)))],\quad \forall T^X\in\T^X,\,T^X\in\T^X.\]
Therefore, for the function $\tilde L_T(\cdot)=\E[L_T(Y,Y')|Y'=\cdot]$, there exists $m\in\R$ such that $\tilde L_T(y)\geq m$ for any $y\in\YCal_T$.

Now fix any $T^Y\in\T^Y$. The continuity of $\tilde L_T$ and the continuity of $T^Y(x,\cdot)$ for each $x\in\XCal_T$ guarantee the continuity of $\tilde L_T(T^y(x,\cdot))$. Together with the compactness of $f_S^*(\XCal_S)$, we have that for any $x\in\XCal_T$, 
\[M_x=\argmin_{y\in f_S^*(\XCal_S)}\tilde L_T(T^Y(x,y))\neq\emptyset.\]
Therefore, for any $T^Y$, one can construct $\tilde T^{X}\in\T^X$ such that $\tilde T^X(x)\in M^x_{T^Y}$ for any $x\in\XCal_T$ and 
\[\min_{T^X\in\T^X}\LL_T(T^Y(\cdot,(f_S^*\circ T^X)(\cdot)))=\E[\tilde L_T(\tilde T^Y(X_T))]=:\tilde\LL_T(\tilde T^Y).\]
The continuity of the new loss functional $\tilde\LL_T$ comes from the continuity of the function $\tilde L$, and the particular choice of the function space $(\{f|f:{\XCal_T}\to\YCal_T\},\|\cdot\|_{\infty})$, where $\{f|f:{\XCal_T}\to\YCal_T\}$ contains all functions from $\XCal_T$ to $\YCal_T$. Since $\tilde\T^Y$ is compact in $(\{f|f:{\XCal_T}\to\YCal_T\},\|\cdot\|_{\infty})$, the minimum over $\tilde\T^Y$ is attained at some $\tilde T^{Y,*}$. According to the definition of $\tilde\T^Y$, there exists $T^{Y,*}\in\T^Y$ such that $\tilde T^{Y,*}(\cdot)=\inf_{y\in f_S^*(\XCal_S}T^{Y,*}(\cdot,y)$. Let $T^{X,*}$ be the $\tilde T^X\in\T^X$ corresponding to $T^{Y,*}$. For any $T^X\in\T^X$ and $T^Y\in\T^Y$, we have
\[\begin{aligned}
    \LL_T(T^Y(\cdot,(f_S^*\circ T^X)(\cdot)))&\geq \LL_T(T^Y(\cdot,(f_S^*\circ \tilde T^X)(\cdot)))\\
    &=\tilde \LL_T(\tilde T^Y(\cdot))\geq \tilde\LL_T(\tilde T^{Y,*}(\cdot))\\
    &=\LL_T(T^{Y,*}(\cdot,(f_S^*\circ T^{X,*}))(\cdot))\geq \min_{T^X\in\mathbb{T}^X,T^Y\in\mathbb{T}^Y}\LL_T\left(T^Y(\cdot, (f_S^*\circ T^X)(\cdot))\right).
\end{aligned}\]
Therefore, the transfer learning problem \eqref{eq: doub-trans} is well-defined and it attains its minimum at $(T^{X,*},T^{Y,*})$ described above. 
\end{proof}
If one removes the compactness assumptions in Assumption (A), then a sufficiently rich family of output transport mappings is needed, 
such that the target optimal model $f_T^*$ can be written as
$f_T^*(x)=T^Y(x,f_S^*(T^X(x))),\quad\forall x\in\XCal_T.$
However, it is often difficult to verify if the set $\T^Y$ is sufficiently rich, due to the construction of neural networks as well as the choices of optimization algorithms.  The compactness conditions, on the other hand, can be implemented through choosing a particular family of activation functions or imposing boundaries restrictions to weights and biases when constructing machine learning models.

\subsection{Properties of Transfer risk}\label{subsec: property}

In this section, the mathematical properties of transfer risk \eqref{eqn:risk} will be studied under mild assumptions. In the following discussion, we will fix a target task $T$ and explore how transfer risk is affected by the choice of source task $S$.

There are two vital pieces of information obtained from the source task $S$ based on the optimization problem in \eqref{eq: doub-trans}, and  transfer risks in \eqref{eqn:risk}. One is the probability distribution of source input $Law(X_S)$, and the other is the pretrained model $f_S^*$ in \eqref{eq: obj-s}  Therefore, we can characterize source task $S$ by $(Law(X_S),f_S^*)$. More specifically, given a  target task $T$ and the source input and output spaces $\XCal_S$ and $\YCal_S$, we can define a corresponding set of pretrained source tasks $\SCal\subset\mathcal{P}(\XCal_S)\times A_S$. Without ambiguity on the target task $T$, we denote $\CCal(S,T)=\CCal(S)=\CCal(\mu,f)$ for any $S=(\mu,f)\in\SCal$. For the set of probability measures over $\XCal_S$, $\mathcal{P}(\XCal_S)$, we can adopt a metric function $D:\mathcal{P}(\XCal_S)\times \mathcal{P}(\XCal_S)\to\R$. Then for the set of functions $A_S$, fix a sufficiently large constant $M>0$ and define the following metric: $\forall f_1,f_2\in A_S,$
\[d_M(f_1,f_2):=\min\{M,\sup_{x\in\XCal_S}\|f_1(x)-f_2(x)\|_{\YCal_S}\}.\]
Then for any $S_1,S_2\in\SCal$ such that $S_1=(\mu_1,f_1)$ and $S_2=(\mu_2,f_2)$, define
\begin{equation}\label{eqn:d_s}
    d_S(S_1,S_2):=D(\mu_1,\mu_2)+d_M(f_1,f_2).
\end{equation}

It is easy to verify that $d_S$ is a metric over $\SCal$. 

In the following discussion on continuity, the next assumption is necessary. Assumption \ref{ass:D} ensures that the choice of input transfer risk is consistent with the metric $d_S$ in \eqref{eqn:d_s} defined between source tasks.

\begin{asp}\label{ass:D}
    For any input transport mapping $T^X_0\in\mathbb{T}^X_0$, assume the input transfer risk $\Ef^I(T^X_0)$ take the form  $\Ef^I(T^X_0):=D(T^X_0\#Law(X_T),Law(X_S))$, where $D:\mathcal{P}(\XCal_S)\times \mathcal{P}(\XCal_S)\to\R$ is the distance function appearing in \eqref{eqn:d_s}.
\end{asp}

By definition, the following degenerate case holds immediately.
\begin{prop}[Zero transfer risk]
    Suppose $\XCal_T=\XCal_S$, $\YCal_T=\YCal_S$ and the target task $T\in\SCal$. Then $\CCal(T)=0$.
\end{prop}

Now, we consider source tasks $S_1,S_2\in\SCal$ that differ only in the input distribution, i.e., $S_f^1=(\mu_1,f)$ and $S_f^2=(\mu_2,f)$. Then we have the following continuity property for $\CCal$.
\begin{prop}[Continuity in input distribution]\label{lem: cont-distr}
    Assume Assumption \ref{ass:D}. Fix $f\in A_S$. $\CCal(\cdot, f)$ is continuous on $(\mathcal{P}(\XCal_S),D)$.
\end{prop}

\begin{proof}[Proof of Proposition \ref{lem: cont-distr}]
Fix an arbitrary $\epsilon>0$. Take any $\mu\in\mathcal{P}(\XCal_S)$.
Then we first establish the lower semi-continuity: For any $T^X_0\in\T_0^X$ and $T_0^Y\in\T_0^Y$, let $f_{I}$ denote the corresponding intermediate model from source model $f$. By Definition \ref{defn: trans-bene}, we have
\[\CCal(\mu,f)-\frac{1}{2}\epsilon<\CCal(\mu,f|f_{I}).\]
By triangle inequality of $D$ and the Lipschitz property of $\CCal(\mu,f|f_I)$, take $\delta=\frac{\epsilon}{2L}$ for any $\mu'\in B_\delta(\mu)\subset \mathcal{P}(\XCal_S)$,
\[\CCal(\mu,f|f_{I})\leq \CCal(\mu',f|f_{I})+L\delta.\]
Notice that the choice of $\delta$ is independent of $T_0^X$ and $T_0^Y$. Therefore, 
\[\CCal(\mu,f)-\epsilon<\CCal(\mu',f;f_{I})\Rightarrow \CCal(\mu,f)-\epsilon<\CCal(\mu',f).\]

Now we show the upper semi-continuity.  From the infimum nature of $\CCal$, there exists $\bar T_0^X\in\T_0^X$ and $\bar T_0^Y\in\T_0^Y$, with corresponding intermediate model $\bar f_I$, such that 
\[\CCal(\mu,f|\bar f_{I})<\CCal(\mu,f)+\frac{1}{2}\epsilon.\]
Again, by triangle inequality of $D$ and the Lipschitz property of $\CCal(\mu,f|f_I)$, take $\delta=\frac{\epsilon}{2L}$ for any $\mu'\in B_\delta(\mu)\subset \mathcal{P}(\XCal_S)$,
\[\CCal(\mu,f|\bar f_{I})\geq \CCal(\mu',f|\bar f_{I})-\delta.\]
Then we have 
\[\CCal(\mu',f)\leq \CCal(\mu',f|\bar f_{I})<\CCal(\mu,f)+\epsilon.\]
Hence, we conclude that $\CCal(\cdot, f)$ is continuous on $(\mathcal{P}(\XCal_S),D)$.
\end{proof}

This proposition shows that transfer risk will change continuously along with any modification in source input. Its proof indicates that the sensitivity of transfer risk with respect to the change in source input distribution depends on the Lipschitz constant $L$ of $\CCal$. Therefore, one can modulate this sensitivity by carefully designing the $C$ function in Definition \ref{defn: trans-bene}. For instance, for linear transfer risk $\CCal^\lambda$ in \eqref{eq: risk-ln}, the sensitivity can be controlled by varying the value of $\lambda$.

Next, consider source tasks $S_1,S_2\in\SCal$ that differ only in the pretrained model, i.e., $S_\mu^1=(\mu,f_1)$ and $S_\mu^2=(\mu,f_2)$. Then we have  the robustness of the transferability in terms of  the continuity of $\CCal(\mu,\cdot)$ in pretrained model $f\in(A_S,d_M)$. 

\begin{prop}[Continuity in pretrained model]\label{lem: cont-f} 
Assume Assumption \ref{ass:D}, and assume that there exists a constant $L>0$ such that for any $T_0^Y\in\T_0^Y$, 
\[T_0^Y(x_1,y_1)-T_0^Y(x_2,y_2)\leq L\left(\|x_1-x_2\|_{\XCal_T}+\|y_1-y_2\|_{\YCal_S}\right),\]
for all $(x_1,y_1),(x_2,y_2)\in\XCal_T\times\YCal_S$.
Assume also that there exist  some $L'>0$ and $p\geq1$ such that the output transfer risk satisfies
    \[\left|\Ef^O(h_1)-\Ef^O(h_2)\right|\leq L'\mathcal{W}_p(h_1\#Law(X_T),h_2\#Law(X_T))^p\]
    for all $h_1,h_2\in\mathcal I$. 
    Then $\CCal(\mu,\cdot)$ is continuous on $(A_S,d_M)$ for any fixed $\mu\in\mathcal{P}(\XCal_S)$.
\end{prop}

\begin{proof}[Proof of Proposition \ref{lem: cont-f}]
Take any $T_0^X\in\T_0^X$ and $T_0^Y\in\T_0^Y$. For any $f_1,f_2\in(A_S,d_M)$, denote their corresponding intermediate model as $f_I^1$ and $f_I^2$, respectively. Then we have
\[\begin{aligned}
|\Ef^O(f_I^1)-\Ef^O(f_I^2)|&\leq L'W_p(f_I^1\#Law(X_T),f_I^2\#Law(X_T))^p\\
&=L'\inf_{\pi\in\Pi(f_I^1\#Law(X_T),f_I^2\#Law(X_T))}\int_{\YCal_T\times\YCal_T}\|x-y\|_{\YCal_T}^p\pi(dx,dy)\\
&\leq L'\inf_{\gamma\in\Pi(Law(X_T),Law(X_T))}\int_{\XCal_T\times\XCal_T}\|T_0^Y(x,f_1(T_0^X(x)))-T_0^Y(y,f_2(T_0^X(y)))\|_{\YCal_T}^2\pi(dx,dy)\\
&\leq 2^{p-1}L^pL'\left[\inf_{\gamma\in\Pi(Law(X_T),Law(X_T))}\int_{\XCal_T\times\XCal_T}\|x-y\|_{\XCal_T}^pd\pi(dx,dy)+d_M(f_1,f_2)^p\right]\\
&=2^{p-1}L^pL'\left[W_p(Law(X_T),Law(X_T))^p+d_M(f_1,f_2)^p\right]=2^{p-1}L^pd_M(f_1,f_2)^p.
\end{aligned}\]
The rest of the proof is similar to that of Proposition \ref{lem: cont-distr}.
\end{proof}

This proposition shows that transfer risk will change continuously along with the modification in the pretrained model. As seen from the proof, the sensitivity of transfer risk with respect to the change in pretrained model is determined by three factors: (1) the Lipschitz constant inherited from the $C$ function in Definition \ref{defn: trans-bene}, (2) the choice of output transport risk $\Ef^O$, and (3) the family of output transport mappings $\T_0^Y$. In practice, one may control the sensitivity of the transfer risk through careful choices of those quantities.


Propositions \ref{lem: cont-distr} and \ref{lem: cont-f} lead to the following results.
\begin{prop}
    \label{thm: cont}
    Suppose the conditions in Proposition \ref{lem: cont-f} hold. Then the transfer risk $\CCal$ as in Definition \ref{defn: trans-bene} is continuous on $(\SCal, d_S)$.
\end{prop}

Propositions \ref{lem: cont-distr} -- \ref{thm: cont} reveals that under a given target task, transfer risk is continuously influenced by both the changes in the source input and the pretrained model. Therefore, transfer risk is to evaluate the suitability of performing transfer learning and the  appropriate choice of given source tasks for a  target task.

\section{Tranfer Risk and Regret with Gaussian Data}\label{app:gaussian}
In this section, we will revisit the example in Section \ref{subsec:gaussian}. The proof of Proposition \ref{prop: lb} will also be presented in this section. In the following discussion, for any spaces $\XCal$ and $\YCal$, we use the notation $\YCal^{\XCal}$ to denote the set of all the functions from $\XCal$ to $\YCal$.

More specifically, consider a transfer learning problem in linear regression where the source and target data are sampled from two Gaussian distributions respectively. 

\subsection{Basic case}
Let us first focus on the case where both data sources are of the same dimension. For the source task $S$, the input and the output spaces are $\XCal_S=\R^d$ and $\YCal_S=\R$, respectively. The source data $(X_S,Y_S)\in\XCal_S\times\YCal_S$ is Gaussian distributed such that $(X_S,Y_S)\sim N(\mu_S,\Sigma_S)$ with
\begin{equation}\label{eq: source-data-distrn}
\mu_S=\begin{pmatrix}\mu_{S,X}\\\mu_{S,Y}\end{pmatrix},\quad \Sigma_S=\begin{pmatrix}\Sigma_{S,X}&\Sigma_{S,XY}\\\Sigma_{S,YX}&\Sigma_{S,Y}\end{pmatrix},
\end{equation}
where
$\mu_{S,Y}\text{ and }\Sigma_{S,Y}\in\R$, $\mu_{S,X}\text{ and }\Sigma_{S,XY}\in\R^d$, $\Sigma_{S,YX}=\Sigma_{S,XY}^\top$, and $\Sigma_{S,X}\in\R^{d\times d}$. Take the set of admissible source models $A_S$ to be the set of functions $f:\R^d\mapsto\R$. For any $f\in A_S$, let the source loss function be 
\begin{equation}\label{eq: lr-loss}
    \LL_S(f)=\E\|Y_S-f(X_S)\|_2^2.
\end{equation}
Then the optimal source model 
\begin{equation}\label{eq: lr-opt-prob}
f_S^*\in\argmin_{f\in A_S}\LL_S(f)
\end{equation}
is given by 
\begin{equation}\label{eq: source-opt-model}
    f_S^*(x)=w_S^\top x+b_s,
\end{equation}
 where
 \begin{equation}\label{eq: pretrn-model}
        w_S=\Sigma_{S,X}^{-1}\Sigma_{S,XY}\in\R^d,\quad b_S=\mu_{S,Y}-\Sigma_{S,YX}\Sigma_{S,X}^{-1}\mu_{S,X}\in\R.
    \end{equation}
    
Such$f_S^*$ is then used as the pretrained model for the following target task $T$, where the target input and output spaces are the same as in the source task, $\XCal_T=\XCal_S$ and $\YCal_T=\YCal_S$. The target data $(X_T,Y_T)\in\XCal_T\times\YCal_T$ follows a different Gaussian distribution from that in the source data such that $(X_T,Y_T)\sim N(\mu_T,\Sigma_T)$, with
\begin{equation}\label{eq: target-data-distrn}
    \mu_T=\begin{pmatrix}\mu_{T,X}\\\mu_{T,Y}\end{pmatrix},\quad \Sigma_T=\begin{pmatrix}\Sigma_{T,X}&\Sigma_{T,XY}\\\Sigma_{T,YX}&\Sigma_{T,Y}\end{pmatrix},
\end{equation}
where $\mu_{T,Y}\text{ and }\Sigma_{T,Y}\in\R$, $\mu_{T,X}\text{ and }\Sigma_{T,XY}\in\R^d$, $\Sigma_{T,YX}=\Sigma_{T,XY}^\top$, and $\Sigma_{T,X}\in\R^{d\times d}$. 

The set of admissible target models is the same as the in the source task such that $A_T=A_S$. For any $f\in A_T$, let the target loss function be $\LL_T(f)=\E\|Y_T-f(X_T)\|_2^2$. Then similarly to the source task, the optimal target model $f_T^*$ is given by
\begin{equation}\label{eq: opt-target-model}
    f_T^*(x)=w_T^\top x+b_T,\quad \forall x\in\R^d,
\end{equation}
where
\begin{equation}\label{eq: target-opt-param}
    w_T=\Sigma_{T,X}^{-1}\Sigma_{T,XY},\quad b_T=\mu_{T,Y}-\Sigma_{T,YX}\Sigma_{T,X}^{-1}\mu_{T,X}.
\end{equation}
The corresponding output distribution is then given by
\begin{equation}\label{eq: target-pred}
\pbm_T=\E[Y|X]=N(w_T^\top\mu_{T,X}+b_T,w_T^\top\Sigma_{T,X}w_T)=N(\mu_T,w_T^\top\Sigma_{T,X}w_T).
\end{equation}

To initiate transfer learning from the source task $S$ to the target task $T$, consider the sets of input and output transport mappings $\T^X=\{f|f:\XCal_T\to\XCal_S\}$ and $\T^Y=\{f|f:\XCal_T\times\YCal_S\to\YCal_T\}$, with corresponding sets of initial transport mappings $\T^X_0=\{id_{\XCal_T}\}$, $\T^Y_0=\{id_{\YCal_S}\}$. Then the set of intermediate models $\ICal$ is a singleton with $\ICal=\{f_{ST}:f_{ST}=f_S^*\}$.

Given the optimal models in both the source task and the target task, specified by \eqref{eq: source-opt-model}-\eqref{eq: pretrn-model} and \eqref{eq: opt-target-model}-\eqref{eq: target-opt-param}, since the data distribution in the target task is given by \eqref{eq: target-data-distrn}, we have
\begin{equation}\label{eq: tl-init-distrn}
        \pbm_{ST}=f_S^*\# N(\mu_{T,X},\Sigma_{T,X})=N(w_S^\top\mu_{T,X}+b_S, w_S^\top\Sigma_{T,X}w_S).
\end{equation}
Notice that $\pbm_T\ll\pbm_{ST}$, therefore the Lebesgue decomposition leads to 
    $\pbm_T=\tilde\pbm_T$, such that
\begin{equation}
    \label{eq: leb-decomp}
    \frac{d\tilde\pbm_T(y)}{d\pbm_{ST}(y)}=h_{ST}(y)=\sqrt\frac{w_S^\top\Sigma_{T,X}w_S}{w_T^\top\Sigma_{T,X}w_T}\exp\left\{\frac{[y-(w_S^\top\mu_{T,X}+b_S)]^2}{2w_S^\top\Sigma_{T,X}w_S}-\frac{[y-(w_T^\top\mu_{T,X}+b)]^2}{w_T^\top\Sigma_{T,X}w_T}\right\}.
\end{equation}
Direct computation leads to the following result.
\begin{itemize}
\item  The KL-based output transfer risk is given by
    \[\begin{aligned}
        \Ef^O_{KL}(f_{ST})=&\frac{1}{2}\left\{\frac{\Sigma_{T,YX}\Sigma_{T,X}^{-1}\Sigma_{T,XY}}{\Sigma_{S,YX}\Sigma_{S,X}^{-1}\Sigma_{T,X}\Sigma_{S,X}^{-1}\Sigma_{S,XY}}-\log{\frac{\Sigma_{T,YX}\Sigma_{T,X}^{-1}\Sigma_{T,XY}}{\Sigma_{S,YX}\Sigma_{S,X}^{-1}\Sigma_{T,X}\Sigma_{S,X}^{-1}\Sigma_{S,XY}}}-1\right.\\
        &\hspace{35pt}\left.{}+\frac{\left[\mu_{T,Y}-\mu_{S,Y}-\Sigma_{S,YX}\Sigma_{S,X}^{-1}\left(\mu_{T,X}-\mu_{S,X}\right)\right]^2}{\Sigma_{S,YX}\Sigma_{S,X}^{-1}\Sigma_{T,X}\Sigma_{S,X}^{-1}\Sigma_{S,XY}}\right\}.
    \end{aligned}\]
\item  The Wasserstein-based output transfer risk is given by
    \[\begin{aligned}
        \Ef^O_{W}(f_{ST})&=\left[\mu_{T,Y}-\mu_{S,Y}-\Sigma_{S,YX}\Sigma_{S,X}^{-1}\left(\mu_{T,X}-\mu_{S,X}\right)\right]^2\\
        &+\left(\sqrt{\Sigma_{S,YX}\Sigma_{S,X}^{-1}\Sigma_{T,X}\Sigma_{S,X}^{-1}\Sigma_{S,XY}}-\sqrt{\Sigma_{T,YX}\Sigma_{T,X}^{-1}\Sigma_{T,XY}}\right)^2.
    \end{aligned}\]
\end{itemize}
The computation shows that
\begin{itemize}
\item The risk in transfer learning is due to the discrepancy in the data distributions between source and target tasks, even when the source and target data are of matching dimensions and follow the same family of distributions. 

\item In particular, in both the KL and the Wasserstein cases, the output transfer risk can be decomposed into two parts, one being  the variance terms $error_{v,\cdot}$ determined by the covariance matrices of the source and target data, and the other being the bias terms $error_{b,\cdot}$ dependent on the difference between  the expectations of $\mu_T$ and $\mu_S$.  

To see this, write
\begin{align}
        \Ef^O_{KL}(f_{ST})&=error_{v,KL}(S,T)+error_{b,KL}(S,T),\label{eq: kl-decomp}\\
        \Ef^O_{W}(f_{ST})&=error_{v,W}(S,T)+error_{b,W}(S,T),\label{eq: w-decomp}
\end{align}
where
\[\begin{aligned}
    & error_{v,KL}(S,T)=\frac{1}{2}\left(\frac{\Sigma_{T,YX}\Sigma_{T,X}^{-1}\Sigma_{T,XY}}{\Sigma_{S,YX}\Sigma_{S,X}^{-1}\Sigma_{T,X}\Sigma_{S,X}^{-1}\Sigma_{S,XY}}-\log{\frac{\Sigma_{T,YX}\Sigma_{T,X}^{-1}\Sigma_{T,XY}}{\Sigma_{S,YX}\Sigma_{S,X}^{-1}\Sigma_{T,X}\Sigma_{S,X}^{-1}\Sigma_{S,XY}}}-1\right),\\
    & error_{b,KL}(S,T)=\frac{\left[\mu_{T,Y}-\mu_{S,Y}-\Sigma_{S,YX}\Sigma_{S,X}^{-1}\left(\mu_{T,X}-\mu_{S,X}\right)\right]^2}{2\Sigma_{S,YX}\Sigma_{S,X}^{-1}\Sigma_{T,X}\Sigma_{S,X}^{-1}\Sigma_{S,XY}};\\
    & error_{v,W}(S,T)=\left(\sqrt{\Sigma_{S,YX}\Sigma_{S,X}^{-1}\Sigma_{T,X}\Sigma_{S,X}^{-1}\Sigma_{S,XY}}-\sqrt{\Sigma_{T,YX}\Sigma_{T,X}^{-1}\Sigma_{T,XY}}\right)^2,\\
    & error_{b,W}(S,T)=\left[\mu_{T,Y}-\mu_{S,Y}-\Sigma_{S,YX}\Sigma_{S,X}^{-1}\left(\mu_{T,X}-\mu_{S,X}\right)\right]^2.
\end{aligned}\]

\item The KL-based variance term 
\[error_{v,KL}=h\left(\frac{\Sigma_{T,YX}\Sigma_{T,X}^{-1}\Sigma_{T,XY}}{\Sigma_{S,YX}\Sigma_{S,X}^{-1}\Sigma_{T,X}\Sigma_{S,X}^{-1}\Sigma_{S,XY}}\right),\]
with the function $h:(0,\infty)\to\R$ such that $h(x)=\frac{1}{2}(x-\log x-1)$ for any $x>0$, which is strictly convex and reaches its minimum value $0$ at $x=1$. Thus,  for both the KL- and the Wasserstein-based output transfer risks, their variance risk components vanish if and only if
    \[\Sigma_{T,YX}\Sigma_{T,X}^{-1}\Sigma_{T,XY}=\Sigma_{S,YX}\Sigma_{S,X}^{-1}\Sigma_{T,X}\Sigma_{S,X}^{-1}\Sigma_{S,XY}.\]
    
\item The bias risk components $error_{b,KL}(S,T)$ and $error_{b,W}(S,T)$ remain strictly positive unless the weighted difference between the expectations $\mu_T$ and $\mu_S$ is $0$.
\end{itemize}

\paragraph{Regret analysis.}
By direct computation, one can show that the regret \eqref{eqn:regret} for this linear transfer leaning problem is given by
\begin{equation}\label{eqn:regret_explicit}
    \RCal(S,T)=\|\Sigma^{\frac{1}{2}}(w_T-w_S)\|_2^2+\left[\mu_{T,Y}-\mu_{S,Y}-\Sigma_{S,YX}\Sigma_{S,X}^{-1}\left(\mu_{T,X}-\mu_{S,X}\right)\right]^2.
\end{equation}
We denote the first term in \eqref{eqn:regret_explicit} as $\hat{error}_v(S,T):=\|\Sigma^{\frac{1}{2}}(w_T-w_S)\|_2^2$, and denote the second term in \eqref{eqn:regret_explicit} as $\hat{error}_b(S,T):=\left[\mu_{T,Y}-\mu_{S,Y}-\Sigma_{S,YX}\Sigma_{S,X}^{-1}\left(\mu_{T,X}-\mu_{S,X}\right)\right]^2$.

Recall from \eqref{eqn:linear_risk} that the Wasserstein-based transfer risk for this problem is defined as $\CCal_{W}(S,T)=\Ef^O_{W}(f_{ST})$ in \eqref{eq: w-decomp}. Meanwhile, it can be easily verified by comparing \eqref{eq: w-decomp} and \eqref{eqn:regret_explicit} that 
\begin{equation}\label{eqn:compare_risk_regret}
    \RCal(S,T)=\CCal_{W}(S,T)+2\left(\|\Sigma_{T,X}^{1/2}w_T\|_2\|\Sigma_{T,X}^{1/2}w_S\|_2-\langle\Sigma_{T,X}^{1/2}w_T,\Sigma_{T,X}^{1/2}w_S\rangle\right).
\end{equation}
Proposition \ref{prop: lb} is an immediate consequence of \eqref{eqn:compare_risk_regret} and the Cauchy–Schwarz inequality.

\begin{rmk}
    Proposition \ref{prop: lb} suggests that for evaluating a transfer learning scheme as in \eqref{eq: tl-fw}, transfer risk provides a proper initial indication of its effectiveness, especially when eliminating unsuitable candidate pretrained models or source tasks if the transfer risk is large. Further examining the decomposition of the transfer $\CCal_{KL}$ and $\CCal_{W}$ as well as the regret $\RCal$, we notice that
    \begin{itemize}
        \item A vanishing bias term in transfer risks is equivalent to a vanishing bias term in regret, i.e., $\hat{error}_{b}(S,T)=0\Longleftrightarrow error_{b,KL}(S,T)=error_{b,W}(S,T)=0$
        \item A vanishing variance term in transfer risk is necessary for a vanishing variance term in regret, i.e., $\hat{error}_{v}(S,T)=0\Longrightarrow error_{v,KL}(S,T)=error_{v,W}(S,T)=0$. 
        \item The residual term $2\left(\|\Sigma_{T,X}^{\frac{1}{2}}w_T\|_2\|\Sigma_{T,X}^{\frac{1}{2}}w_S\|_2-\langle\Sigma_{T,X}^{\frac{1}{2}}w_T,\Sigma_{T,X}^{\frac{1}{2}}w_S\rangle\right)$ in \eqref{eqn:compare_risk_regret} depends entirely on the source and target covariance matrices $\Sigma_S$ and $\Sigma_T$is due to the variance term in the learning objective difference. Therefore, when $\CCal_{W}(S,T)=0$ (or $\CCal_{KL}(S,T)=0$), the training process is to reduce the angular distance between $\Sigma_{T,X}^{1/2}w_S$ and $\Sigma_{T,X}^{1/2}w_T$ caused by the discrepancy in these two covariance matrices.
    \end{itemize}
\end{rmk}

\subsection{Case with feature augmentation} Let us now consider the  case with feature augmentation. That is, compared with the input data in the source task, the target task includes more input information in the form of a higher input dimension. 
We will see that potential extra transfer risk as a result of the extra augmented  input information as well as its benefit to eliminate the bias risk.

Take the same source task $S$ as in the basic case; for the target task $T$, let the input space  $\XCal_T=\R^{d+k}$ with $k\in\mathbb{N}^+$, let  the output space be the same as in the target task $T$ such that $\YCal_T=\YCal_S=\R$. Since the transfer learning problem has a feature augmentation, let us first define a projection  $\XCal_T$ from $T^X_0$ to $\XCal_S$ such that 
\[T^X_0(x)=\begin{pmatrix}I_d&\vdots&O_{d\times k}\end{pmatrix}x,\quad\forall x\in\XCal_T.\]
Then the target data $(X_T,Y_T)\in\XCal_T\times\YCal_T$ satisfies that
$T^X_0(X_T)=X_S$ and $Y_T=Y_S$. That is, $(X_T,Y_T)$ is given by a Gaussian distribution $N(\mu_T,\Sigma_T)$ with $\mu_T$ and $\Sigma_T$ in the same form as in \eqref{eq: target-data-distrn}, where
\[
\mu_{T,X}=\begin{pmatrix}\mu_{S,X}\\\mu_{A,X}\end{pmatrix},\,  \mu_{T,Y}=\mu_{S,Y};\ \ 
    \Sigma_{T,X}=\begin{pmatrix}\Sigma_{S,X} & \Sigma_{AS,X}\\\Sigma_{AS,X}^\top & \Sigma_{A,X}\end{pmatrix},\, \Sigma_{T,XY}=\begin{pmatrix}\Sigma_{S,XY}\\\Sigma_{A,XY}\end{pmatrix},\, \Sigma_{T,Y}=\Sigma_{S,Y}.
\]
Here $\mu_{A,X}\in\R^k$ denotes the expectation of the augmented variable $\tilde X_T$ such that $X_T^\top=\begin{pmatrix}
    T^X_0(X_T)^\top&\tilde X_T^\top
\end{pmatrix}$; in the above covariance matrix $\Sigma_T$, $\Sigma_{AS,X}=Cov(X_S,\tilde X_T)\in\R^{d\times k}$, $\Sigma_{A,X}=Var(\tilde X_T)\in\R^{k\times k}$, and $\Sigma_{A,XY}=Cov(\tilde X_T,Y_T)\in\R^{k}$.
The optimal linear model $f_T^*$ is again given by \eqref{eq: opt-target-model}-\eqref{eq: target-opt-param} with the optimal parameters $w_T$ and $b_T$ re-computed under the above modified target data distribution. The corresponding output distribution $\pbm_T$ is of the form \eqref{eq: target-pred} with updated parameters as in $f_T^*$.

To initialize the transfer learning problem from $S$ to $T$, consider $\T^X_0=\{T^X_0\}$, $\T^Y_0=\{id_{\YCal_S}\}$, $\T^X=\{f|f:{\XCal_T}\to\XCal_S\}$, and $\T^Y=\{f|f:\XCal_T\times\YCal_S\to\YCal_T\}$. The set of intermediate models $\ICal$ is still singleton, with $\ICal=\{f_{ST}:f_{ST}=f_S^*\circ T^X_0\}$. Clearly,  $\pbm_{ST}=Law(f_S^*(X_S))$, with $\pbm_T\ll Law(f_S^(X_S))$.
Now we have
  \begin{itemize}
  \item The KL-based output transfer risks are given by 
    \begin{equation}\label{eq: kl-or-ia}
        \Ef^O_{KL}(f_{ST})=\frac{1}{2}\left[\frac{\Sigma_{T,YX}\Sigma_{T,X}^{-1}\Sigma_{T,XY}}{\Sigma_{S,YX}\Sigma_{S,X}^{-1}\Sigma_{S,XY}}-\log{\frac{\Sigma_{T,YX}\Sigma_{T,X}^{-1}\Sigma_{T,XY}}{\Sigma_{S,YX}\Sigma_{S,X}^{-1}\Sigma_{S,XY}}}-1\right];
\end{equation}
\item The  Wasserstein-based output transfer risk is 
\begin{equation}
        \Ef^O_{W}(f_{ST})=\left(\sqrt{\Sigma_{T,YX}\Sigma_{T,X}^{-1}\Sigma_{T,XY}}-\sqrt{\Sigma_{S,YX}\Sigma_{S,X}^{-1}\Sigma_{S,XY}}\right)^2.
    \end{equation}
\end{itemize}

Comparing the basic case and  this feature augmentation case, we see
\begin{itemize}
\item The extra input information enables the particular choice of the initial input and output transport mappings, $T^X_0$ and $id_{\YCal_S}$, which in turn  eliminates the bias risk component in both the KL- and Wasserstein-based output risk. 
\item Both output transfer risks come from their corresponding variance risk component. Take the KL-based output transfer risk in \eqref{eq: kl-or-ia} as an example. We see that 
\[\Ef^O_{KL}(f_{ST})=error_{v,KL}(S,T)=h\left(\frac{\Sigma_{T,YX}\Sigma_{T,X}^{-1}\Sigma_{T,XY}}{\Sigma_{S,YX}\Sigma_{S,X}^{-1}\Sigma_{S,XY}}\right).\]
This suggests that the challenge of applying transfer learning with feature augmentation lies mainly at the uncertainty from the augmented variable $\tilde X_T$. 
\item In particular, if one assumes that the added input information $\tilde X_T$ is uncorrelated with the existing input data $X_S$, then 
\[\frac{\Sigma_{T,YX}\Sigma_{T,X}^{-1}\Sigma_{T,XY}}{\Sigma_{S,YX}\Sigma_{S,X}^{-1}\Sigma_{S,XY}}=1+\frac{\Sigma_{A,YX}\Sigma_{A,X}^{-1}\Sigma_{A,XY}}{\Sigma_{S,YX}\Sigma_{S,X}^{-1}\Sigma_{S,XY}}.\]
That is, when one introduces new features that are uncorrelated with the existing ones, the variance risk component is always positive unless these new features are also uncorrelated with the output.

\end{itemize}

\subsection{Case with augmented output space}
Let us now consider the case with an extra prediction task, i.e.,  a transfer learning problem with augmented output space. 
In this case, we will  see extra transfer risk with two major contributing factors, one being the unforeseeable correlation between the input and the extra output information and the other being the necessary initialization procedure due to the extra task in the target task. 

To see this, let us consider a source task  slightly modified from the basic case, where the output space in $S$ is allowed be of dimension bigger that $1$, that is, $\YCal_S=\R^l$ with $l\in\mathbb{N}^+$. Then the source data $(X_S,Y_S)$  is given by a Gaussian distribution $N(\mu_S,\Sigma_S)$ with $\mu_S$ and $\Sigma_S$ as in \eqref{eq: source-data-distrn} except that $\mu_{S,Y}\in\R^l$, $\Sigma_{S,XY}\in\R^{d\times l}$ and $\Sigma_{S,Y}\in\R^{l\times l}$. Again the optimal linear model $f_S^*$  is given by \eqref{eq: source-opt-model}-\eqref{eq: pretrn-model} with optimal parameters $w_S$ and $b_S$ re-computed under the above modified source data distribution. 

For the target task, let $\XCal_T=\XCal_S$ and $\YCal_T=\R^{l+k}$ with $k\in\mathbb{N}^+$. Since the transfer learning problem has an extra learning task with the same input data, the target data $(X_T,Y_T)\in\XCal_T\times\YCal_T$ satisfies that $X_T=X_S$ and $Y_T^\top=\begin{pmatrix}Y_S&Y_A\end{pmatrix}^\top$ from some random variable $Y_A\in\R^K$. Let us assume that the joint distribution of the input and output variables $(X_T,Y_T)$ follows a Gaussian distribution $N(\mu_T,\Sigma_T)$ with $\mu_T$ and $\Sigma_T$ in the same form as in \eqref{eq: target-data-distrn}, where
\[\begin{aligned}
    &\mu_{T,Y}=\begin{pmatrix}\mu_{S,Y}\\\mu_{A,Y}\end{pmatrix},\quad \mu_{T,X}=\mu_{S,X};\\
    &\Sigma_{T,Y}=\begin{pmatrix}\Sigma_{S,Y} & \Sigma_{AS,Y}\\\Sigma_{AS,Y}^\top & \Sigma_{A,Y}\end{pmatrix},\quad \Sigma_{T,YX}=\begin{pmatrix}\Sigma_{S,YX} \\ \Sigma_{A,YX}\end{pmatrix}, \quad \Sigma_{T,XY}=\begin{pmatrix}\Sigma_{S,XY} & \Sigma_{A,XY}\end{pmatrix},\quad \Sigma_{T,X}=\Sigma_{S,X}.
\end{aligned}\]
Here $\mu_{A,Y}=\E[Y_{A}] \in\R^k$ denotes the expectation of $Y_A$, $\Sigma_{AS,Y}=Cov(Y_S,Y_A)\in\R^{l\times k}$, $\Sigma_{A,Y}=Var(Y_A)\in\R^{k\times k}$ and $\Sigma_{A,XY}=\Sigma_{A,YX}^\top=Cov(X_S,Y_A)\in\R^{d\times k}$. 
Then again the optimal linear model $f_T^*$ is in the form of \eqref{eq: opt-target-model}
with parameters
\[\begin{aligned}
    &w_T^\top=\Sigma_{T,YX}\Sigma_{T,X}^{-1}=\begin{pmatrix}\Sigma_{S,YX}\Sigma_{S,X}^{-1}\\\Sigma_{A,YX}\Sigma_{S,X}^{-1}\end{pmatrix}=\begin{pmatrix}w_S^\top\\\Sigma_{A,YX}\Sigma_{S,X}^{-1}\end{pmatrix},\\ &b_T=\mu_{T,Y}-\Sigma_{T,YX}\Sigma_{T,X}^{-1}\mu_{T,X}=\begin{pmatrix}\mu_{S,Y}\\\mu_{A,Y}\end{pmatrix}-\begin{pmatrix}\Sigma_{S,YX}\Sigma_{S,X}^{-1}\mu_{S,X}\\\Sigma_{A,YX}\Sigma_{S,X}^{-1}\mu_{S,X}\end{pmatrix}=\begin{pmatrix}b_S\\\mu_{A,Y}-\Sigma_{A,YX}\Sigma_{S,X}^{-1}\mu_{S,X}\end{pmatrix}.\end{aligned}\]
Correspondingly, $\pbm_T=\E[Y_T|X_T]=N\left(\mu_1,\Sigma_1\right)$, where
\[\begin{aligned}
\mu_1=w_T^\top\mu_{T,X}+b_T=\begin{pmatrix}w_S^\top\mu_{S,X}+b_S\\\mu_{A,Y}\end{pmatrix},\ \
\Sigma_1=w_T^\top\Sigma_{T,X}w_T=\begin{pmatrix}w_S^\top\Sigma_{S,X}w_S & w_S^\top\Sigma_{A,XY}\\\Sigma_{A,YX}w_S & \Sigma_{A,YX}\Sigma_{S,X}^{-1}\Sigma_{A,XY}\end{pmatrix}.
\end{aligned}\]

To initialize the transfer learning, consider the sets of input and output transport mappings $\T^X=\{f|f:\XCal_T\to\XCal_S\}$ and $\T^Y=\{f|f:\XCal_T\times\YCal_S\to\YCal_T\}$, as well as the sets of initial input transport mappings $\T^X_0=\{id_{\XCal_T}\}$. For the initial output mapping, in order to handle the newly added prediction task from $X_T=X_S$ to $Y_A$, let us define an initial function $f_0:\R^d\to\R^k$ as $f_0(x)=w_0^\top x+b_0$ for any $x\in\R^d$ with fixed $w_0\in\R^{k\times d}$ and $b_0\in\R^k$. The set of initial output transport mappings is given by $\T^Y_0=\{T^Y_0:\XCal_T\times\YCal_S\to\YCal_T|T^Y_0(x,y)=\begin{pmatrix}y^\top&\vdots&f_0(x)^\top\end{pmatrix}^\top\}$. Once again, the set of intermediate models $\ICal$ is a singleton with $\ICal=\{f_{ST}: f_{ST}(x)=T^Y_0(x,f_S^*(x)),\,\forall x\in\XCal_T\}$.
The probability distribution $\pbm_{ST}$ of the intermediate model $f_{ST}$ is given by $\pbm_{ST}=N(\mu_2,\Sigma_2)$, where
\[\begin{aligned}    &\mu_2=\begin{pmatrix}w_S^\top\\w_0^\top\end{pmatrix}\mu_{T,X}+\begin{pmatrix}b_S\\b_0\end{pmatrix}=\begin{pmatrix}w_S^\top\mu_{S,X}+b_S\\w_0^\top\mu_{S,X}+b_0\end{pmatrix},\\
    &\Sigma_2=\begin{pmatrix}w_S^\top\\w_0^\top\end{pmatrix}\Sigma_{T,X}\begin{pmatrix}w_S & w_0\end{pmatrix}=\begin{pmatrix}w_S^\top\Sigma_{S,X}w_S& w_S^\top\Sigma_{S,X}w_0\\w_0^\top\Sigma_{S,X}w_S & w_0^\top\Sigma_{S,X}w_0\end{pmatrix}.
\end{aligned}\]
We have again $\pbm_T\ll\pbm_{ST}$, and 
\begin{itemize}
 \item    The KL- 
 and Wasserstein-based output transfer risks are given by 
    \begin{equation}
        \Ef^O_{KL}(f_{ST})=\frac{1}{2}\left[Tr(\Sigma_2^{-1}\Sigma_1)-\log{\frac{\text{det}(\Sigma_1)}{\text{det}(\Sigma_2)}}-(k+l)+(\mu_1-\mu_2)^\top\Sigma_2^{-1}(\mu_1-\mu_2)\right];\label{eq: kl-or-oa}
        \end{equation}
      \item The  Wasserstein-based output transfer risk is  given by
      \begin{equation}
        \Ef^O_{W}(f_{ST})=\|\mu_1-\mu_2\|_2^2+Tr\left(\Sigma_1+\Sigma_2-2\left(\Sigma_1^{\frac{1}{2}}\Sigma_2\Sigma_1^{\frac{1}{2}}\right)^{\frac{1}{2}}\right).
    \end{equation}
\end{itemize}
The analysis shows that with the augmented output space, the output transfer risks vanish if the initialization function $f_0$ can neutralize the uncertainty brought by the correlation between the input $X_S$ and the additional output information $Y_A$. 

To see this, take the example of the KL-based output transfer risk in \eqref{eq: kl-or-oa}, and decompose $\Ef^O_{KL}(f_{ST})$ in \eqref{eq: kl-or-oa} into its variance and bias components as in \eqref{eq: kl-decomp}, with
\[\begin{aligned}
    & error_{v,KL}(S,T)=\frac{1}{2}Tr(\Sigma_2^{-1}\Sigma_1)-\log{\frac{\text{det}(\Sigma_1)}{\text{det}(\Sigma_2)}}-(k+l),\\
    & error_{b,KL}(S,T)=\frac{1}{2}(\mu_1-\mu_2)^\top\Sigma_2^{-1}(\mu_1-\mu_2).
\end{aligned}\]
Now, we see that
\begin{itemize}
    \item If $0<\lambda_1\leq\dots\leq\lambda_{k+l}$ are the eigenvalues of $\Sigma_2^{-1}\Sigma_1$, and if $\Sigma_1$ and $\Sigma_2$ are invertible. Then the variance term can be written as
    \[
        error_v(S,T)=\sum_{i=1}^{k+l}(\lambda_i-\log\lambda_i-1)\geq0, 
    \]
    which vanishes if and only if $\lambda_i$'s are all equal to $1$ such that $\Sigma_1=\Sigma_2$.

    \item The difference between the expectations of $\pbm_T$ and $\pbm_{ST}$  is given by
    \[\mu_1-\mu_2=\begin{pmatrix}0\\\mu_{A,Y}-w_0^\top\mu_{S,X}-b_0\end{pmatrix}.\]
    Therefore, the error between the expected augmented output $\mu_{A,Y}$ and $w_0^\top\mu_{S,X}+b_0$ derived from the chosen initialization $f_0$ is the main contributor to a strictly positive bias risk component $error_{b,KL}$. 
\end{itemize}

\end{document}